\documentclass[letterpaper, 10 pt]{article}
\RequirePackage{zaiweiandci}

\title{\LARGE \bf  Finite-Sample Analysis of Off-Policy TD-Learning via Generalized Bellman Operators}

\author{
{\normalsize Zaiwei Chen}\thanks{Georgia Institute of Technology, {\color{blue}zchen458@gatech.edu}}
\and
{\normalsize Siva Theja Maguluri}\thanks{Georgia Institute of Technology, {\color{blue}siva.theja@gatech.edu }}
\and
{\normalsize Sanjay Shakkottai\thanks{The University of Texas at Austin, {\color{blue}sanjay.shakkottai@utexas.edu}}}
\and
{\normalsize Karthikeyan Shanmugam}\thanks{IBM Research NY, {\color{blue}Karthikeyan.Shanmugam2@ibm.com}}
}

\date{}
\begin{document}
\maketitle

\begin{abstract}
In temporal difference (TD) learning, off-policy sampling is known to be more practical than on-policy sampling, and by decoupling learning from data collection, it enables data reuse. It is known that policy evaluation (including multi-step off-policy importance sampling) has the interpretation of solving a generalized Bellman equation. In this paper, we derive finite-sample bounds for any general off-policy TD-like stochastic approximation algorithm that solves for the fixed-point of this generalized Bellman operator. Our key step is to show that the generalized Bellman operator is simultaneously a contraction mapping with respect to a weighted $\ell_p$-norm for each $p$ in $[1,\infty)$, with a common contraction factor.

Off-policy TD-learning is known to suffer from  high variance due to the product of importance sampling ratios. A number of algorithms (e.g. $Q^\pi(\lambda)$, Tree-Backup$(\lambda)$, Retrace$(\lambda)$, and $Q$-trace) have been proposed in the literature to address this issue. Our results immediately imply finite-sample bounds of these algorithms. In particular, we provide first-known finite-sample guarantees for $Q^\pi(\lambda)$, Tree-Backup$(\lambda)$, and Retrace$(\lambda)$, and improve the best known bounds of $Q$-trace in \citep{chen2021finite}. Moreover, we show the bias-variance trade-offs in each of these algorithms.
\end{abstract}

\section{Introduction}\label{sec:intro}
Reinforcement learning (RL) demonstrated its success in learning  effective policies for a variety of decision making problems such as autonomous driving \citep{shalev2016safe,sallab2017deep}, recommender systems \citep{aggarwal2016recommender,zou2019reinforcement}, and game-related problems \citep{silver2017mastering,zhang2020learning,nowe2012game}. In RL, there is an important sub-problem -- called the policy evaluation problem -- of estimating the expected long term reward of a given policy.

The policy evaluation problem is usually solved with the temporal difference (TD) learning method \citep{sutton1988learning}. A key ingredient in TD-learning is the policy used to collect samples (called the behavior policy). Ideally, we want to generate samples from the target policy whose value function we want to estimate, and this is called on-policy sampling. However, in many cases such on-policy sampling is not possible due to practical reasons, and hence we need to work with historical data that is generated by a possibly different policy (i.e., off-policy sampling). For example, in high stake applications such as clinic trials \citep{zhao2011reinforcement} and health care \citep{gottesman2020interpretable}, it is not practically possible to re-collect data every time we need to evaluate the performance of a chosen policy. Although off-policy sampling is more practical than on-policy sampling, it is more challenging to analyze and is known to have high variance \citep{glynn1989importance}, which is due to the presence of the product of the importance sampling ratios, and is a fundamental difficulty in off-policy learning. 

To overcome this difficulty, many variants of off-policy TD-learning algorithms have been proposed in the literature, such as the $Q^\pi(\lambda)$ algorithm \citep{harutyunyan2016q}, the Tree-Backup$(\lambda)$ (henceforth denoted by TB$(\lambda)$) algorithm \citep{precup2000eligibility}, the Retrace$(\lambda)$ algorithm \citep{munos2016safe}, and the $Q$-trace algorithm \citep{chen2021finite}, etc.

\subsection{Main Contributions}
In this work, we establish finite-sample bounds of a general $n$-step off-policy TD-learning algorithm that also subsumes several algorithms presented in the literature. The key step is to show that such algorithm can be modeled as a Markovian stochastic approximation (SA) algorithm for solving a generalized Bellman equation.
We present sufficient conditions under which the generalized Bellman operator is contractive with respect to a weighted $\ell_p$-norm for every $p\in [1,\infty)$, with a uniform contraction factor for all $p$. Our result shows that the sample complexity scales as $\Tilde{\mathcal{O}}(\epsilon^{-2})$, where $\epsilon$ is the required accuracy \footnote{In the $\Tilde{\mathcal{O}}(\cdot)$ notation, we ignore all the polylogarithmic terms.}. It also involves a factor that depends on the problem parameters, in particular, the generalized importance sampling ratios, and explicitly demonstrates the bias-variance trade-off. 

Our result immediately gives finite-sample guarantees for variants of multi-step off-policy TD-learning algorithms including $Q^\pi(\lambda)$, TB$(\lambda)$, Retrace$(\lambda)$, and $Q$-trace. For $Q^\pi(\lambda)$, TB$(\lambda)$, and Retrace$(\lambda)$, we establish the first-known results in the literature, while for $Q$-trace, we improve the best known results in \citep{chen2021finite} in terms of the dependency on the size of the state-action space. The weighted $\ell_p$-norm contraction property with a uniform contraction factor for all $p\in [1,\infty)$ is crucial for us to establish the improved sample complexity. Based on the finite-sample bounds, we show that all four algorithms overcome the high variance issue in Vanilla off-policy TD-learning, but their convergence rates are all affected to varying degrees.

\subsection{Generalized Bellman Operator and Stochastic Approximation}\label{subsec:intuition}
In this section, we illustrate the interpretation of off-policy multi-step TD-learning as an SA algorithm for solving a generalized Bellman equation. Consider the policy evaluation problem where the goal is to estimate the state-action value function $Q^\pi$ of a given policy $\pi$. See Section \ref{subsec:background} for more details. In the simplest setting where TD$(0)$ with on-policy sampling is employed, it is well known that the algorithm is an SA algorithm for solving the Bellman equation $Q=\mathcal{H}_\pi(Q)$, where $\mathcal{H}_\pi(\cdot)$ is the Bellman operator. The generalized Bellman operator $\mathcal{B}(\cdot)$ we consider in this paper is defined by:
\begin{align}\label{eq:Bellman_general}
    \mathcal{B}(Q)=\mathcal{T}(\mathcal{H}(Q)-Q)+Q,
\end{align}
where $\mathcal{T}(\cdot)$ and $\mathcal{H}(\cdot)$ are two auxiliary operators. In the special case where $\mathcal{T}(\cdot)$ is the identity operator and $\mathcal{H}(\cdot)$ is the Bellman operator $\mathcal{H}_\pi(\cdot)$, the generalized Bellman operator $\mathcal{B}(\cdot)$ reduces to the regular Bellman operator $\mathcal{H}_\pi(\cdot)$. Note that any fixed point of $\mathcal{H}(\cdot)$ is also a fixed point of $\mathcal{B}(\cdot)$, as long as $\mathcal{T}(\cdot)$ is such that $\mathcal{T}(\bm{0})=\bm{0}$. Thus, the operator $\mathcal{H}(\cdot)$ controls the fixed-point of the generalized Bellman operator $\mathcal{B}(\cdot)$, and as we will see later, the operator $\mathcal{T}(\cdot)$ can be used to control its contraction properties. 

To further understand the operator $\mathcal{B}(\cdot)$, we demonstrate in the following that both on-policy $n$-step TD and TD$(\lambda)$ can be viewed as SA algorithms for solving the generalized Bellman equation $\mathcal{B}(Q)=Q$, with different auxiliary operators $\mathcal{T}(\cdot)$ and $\mathcal{H}(\cdot)$.

On-policy $n$-step TD is designed to solve the $n$-step Bellman equation $Q=\mathcal{H}_\pi^n(Q)$. Since the Bellman operator $\mathcal{H}_\pi(\cdot)$ is explicitly given by $\mathcal{H}_\pi(\cdot)=R+\gamma P_\pi(\cdot)$, where $R$ is the reward vector, $\gamma$ is the discount factor, and $P_\pi$ is the transition probability matrix under policy $\pi$, the equation $Q=\mathcal{H}_\pi^n(Q)$ can be written as $Q=\sum_{i=0}^{n-1}(\gamma P_\pi)^iR+(\gamma P_\pi)^nQ$, which by reverse telescoping is equivalent to
\begin{align*}
    Q=\sum_{i=0}^{n-1}(\gamma P_\pi)^i(R+\gamma P_\pi Q-Q)+Q=\mathcal{T}(\mathcal{H}_\pi(Q)-Q)+Q,
\end{align*}
where $\mathcal{T}(Q)=\sum_{i=0}^{n-1}(\gamma P_\pi)^i Q$. Similarly, one can formulate the TD$(\lambda)$ Bellman equation in the form of $\mathcal{B}(Q)=Q$, where $\mathcal{T}(Q)=(1-\lambda)\sum_{i=0}^\infty\lambda^i\sum_{j=0}^{i-1}(\gamma P_\pi)^i Q$ and $\mathcal{H}(\cdot)=\mathcal{H}_\pi(\cdot)$.

In these examples, the operator $\mathcal{T}(\cdot)$ determines the contraction factor of $\mathcal{B}(\cdot)$ by controlling the degree of bootstrapping. In this work, we show that in addition to on-policy TD-learning, variants of off-policy TD-learning with multi-step bootstrapping and generalized importance sampling ratios can also be interpreted as SA algorithms for solving the generalized Bellman equation. Moreover, under some mild conditions, we show that the generalized Bellman operator $\mathcal{B}(\cdot)$ is a contraction mapping with respect to some weighted $\ell_p$-norm for any $p\in [1,\infty)$, with a common contraction factor. This enables us to establish finite-sample bounds of general multi-step off-policy TD-like algorithms. 

\subsection{Related Literature}\label{subsec:literature}
The TD-learning method was first proposed in \citep{sutton1988learning} for solving the policy evaluation problem. Since then, there is an increasing interest in theoretically understanding TD-learning and its variants.

\textit{On-Policy TD-Learning.} The most basic TD-learning method is the TD$(0)$ algorithm \citep{sutton1988learning}. Later it was extended to using multi-step bootstrapping (i.e., the $n$-step TD-learning algorithm \citep{watkins1989learning,cichosz1995fast,van2016true}), and using eligibility trace (i.e., the TD$(\lambda)$ algorithm \citep{sutton1988learning,singh1994learning}). The asymptotic convergence of TD-learning was established in \citep{tsitsiklis1994asynchronous,jaakkola1994convergence,dayan1992convergence}. As for finite-sample analysis, a unified Lyapunov approach is presented in \citep{chen2021lyapunov}. To overcome the curse of dimensionality in RL, TD-learning is usually incorporated with function approximation in practice. In the basic setting where a linear parametric architecture is used, the asymptotic convergence of TD-learning was established in \citep{tsitsiklis1997analysis}, and finite-sample bounds in \citep{bhandari2018finite,srikant2019finite,thoppe2015concentration,dalal2018finite}. Very recently, the convergence and finite-sample guarantee of TD-learning with neural network approximation were studied in \citep{cayci2021sample,cai2019neural}.

\textit{Off-Policy TD-Learning.} In the off-policy setting, since the samples are not necessarily generated by the target policy, usually importance sampling ratios (or ``generalized'' importance sampling ratios) are introduced in the TD-learning algorithm. The resulting algorithms are $Q^\pi(\lambda)$ \citep{precup2000eligibility}, TB$(\lambda)$ \citep{harutyunyan2016q}, Retrace$(\lambda)$ \citep{munos2016safe}, and $Q$-trace \citep{chen2021finite} (which is an extension of $V$-trace \citep{espeholt2018impala}), etc. The asymptotic convergence of these algorithms has been established in the papers in which they were proposed. To the best of our knowledge, finite-sample guarantees are established only for $Q$-trace and $V$-trace \citep{chen2021finite,chen2020finite,chen2021lyapunov}. In the function approximation setting, TD-learning with off-policy sampling and function approximation is a typical example of the deadly triad \citep{sutton2018reinforcement}.  In the presence of the deadly triad, the algorithm can be unstable \citep{sutton2018reinforcement,baird1995residual}. To achieve convergence, one needs to significantly modify the original TD-learning algorithm, resulting in two time-scale algorithms such as GTD \citep{maei2011gradient}, TDC \citep{sutton2009fast}, and emphatic TD \citep{sutton2016emphatic}, etc.

\subsection{Preliminaries}\label{subsec:background}
In this section, we cover the background of RL and the TD-learning method for solving the policy evaluation problem. The RL problem is usually modeled as a Markov decision process (MDP). In this work, we consider an MDP with a finite set of states $\mathcal{S}$, a finite set of actions $\mathcal{A}$, a set of unknown action dependent transition probability matrices $\mathcal{P}=\{P_a\in\mathbb{R}^{|\mathcal{S}|\times|\mathcal{S}|}\mid a\in\mathcal{A} \}$, an unknown reward function $\mathcal{R}:\mathcal{S}\times\mathcal{A}\mapsto[0,1]$, and a discount factor $\gamma\in (0,1)$. 

In order for an MDP to progress, we must specify the policy of selecting actions based on the state of the environment. Specifically, a policy $\pi$ is a mapping from the state-space to probability distributions supported on the action space, i.e., $\pi:\mathcal{S}\mapsto\Delta^{|\mathcal{A}|}$. The state-action value function $Q^\pi$ associated with a policy $\pi$ is defined by $Q^\pi(s,a)=\mathbb{E}_{\pi}[\sum_{k=0}^\infty\gamma^k\mathcal{R}(S_k,A_k)\mid S_0=s,A_0=a]$ for all $(s,a)$. The goal we consider in this paper is to estimate the state-action value function $Q^\pi$ of a given policy $\pi$.

Since the transition probabilities as well as the reward function are unknown, such state-action value function cannot be directly computed. The TD-learning algorithm is designed to estimate $Q^\pi$ using the SA method. Specifically, in TD-learning, we first collect a sequence of samples $\{(S_k,A_k)\}$ from the model using some behavior policy $\pi_b$. Then the value function $Q^\pi$ is iteratively estimated using the samples $\{(S_k,A_k)\}$. When $\pi_b=\pi$, the algorithm is called on-policy TD-learning, otherwise the algorithm is referred to as  off-policy TD-learning.

\section{Finite-Sample Analysis of General Off-Policy TD-Learning}\label{sec:RL}
In this section, we present our unified framework for finite-sample analysis of off-policy TD-learning algorithms using generalized importance sampling ratios and multi-step bootstrapping. The proofs of all technical results presented in this paper are provided in the Appendix.

\subsection{A Generic Model for Multi-Step Off-Policy TD-Learning}\label{subsec:algorithm}
Algorithm \ref{algorithm} presents our generic algorithm model. Due to off-policy sampling, the two functions $c,\rho:\mathcal{S}\times\mathcal{A}\mapsto\mathbb{R}_+$ are introduced in Algorithm \ref{algorithm} to serve as generalized importance sampling ratios in order to account for the discrepancy between the target policy $\pi$ and the behavior policy $\pi_b$. We denote $c_{\max}=\max_{s,a}c(s,a)$ and $\rho_{\max}=\max_{s,a}\rho(s,a)$. We next show how Algorithm \ref{algorithm} captures variants of off-policy TD-learning algorithms in the literature by using different generalized importance sampling ratios $c(\cdot,\cdot)$ and $\rho(\cdot,\cdot)$.

\begin{algorithm}[h]\caption{A Generic Algorithm for Multi-Step Off-Policy TD-Learning}\label{algorithm}
\begin{algorithmic}[1] 
	\STATE {\bfseries Input:} $K$, $\{\alpha_k\}$, $Q_0$, $\pi$, $\pi_b$, generalized importance sampling ratios $c,\rho:\mathcal{S}\times\mathcal{A}\mapsto\mathbb{R}_+$, and sample trajectory $\{(S_k,A_k)\}_{0\leq k\leq K+n}$ collected under the behavior policy $\pi_b$.
	\FOR{$k=0,1,\cdots,K-1$}
	\STATE $\alpha_k(s,a)=\alpha_k\mathbb{I}{{\{(s,a)=(S_k,A_k)\}}}$ for all $(s,a)$ 
	\STATE $\Delta(S_i,A_i,S_{i+1},A_{i+1},Q_k)=\mathcal{R}(S_i,A_i)+\gamma\rho(S_{i+1},A_{i+1})Q_k(S_{i+1},A_{i+1})-Q_k(S_i,A_i)$ for all $i\in \{k,k+1,...,k+n-1\}$. 
	\STATE $Q_{k+1}(s,a)=
	Q_k(s,a)+\alpha_k(s,a)\sum_{i=k}^{k+n-1}\gamma^{i-k}\prod_{j=k+1}^ic(S_j,A_j)\Delta(S_i,A_i,S_{i+1},A_{i+1},Q_k)$ for all $(s,a)$ 
	\ENDFOR
	\STATE\textbf{Output:} $Q_K$
\end{algorithmic}
\end{algorithm}

\begin{itemize}
    \item \textit{Vanilla IS.} When $c(s,a)=\rho(s,a)=\pi(a|s)/\pi_b(a|s)$ for all $(s,a)$, Algorithm \ref{algorithm} is the standard off-policy TD-learning with importance sampling  \citep{precup2000eligibility}. We will refer to this algorithm as Vanilla IS. Although Vanilla IS was shown to converge to $Q^\pi$ \citep{precup2000eligibility}, since the product of importance sampling ratios $\prod_{j=k+1}^i\pi(A_j|S_j)/\pi_b(A_j|S_j)$ is not controlled in any way, it suffers the most from high variance.
    \item \textit{The $Q^\pi(\lambda)$ Algorithm.} When $c(s,a)=\lambda $ and $\rho(s,a)=\pi(a|s)/\pi_b(a|s)$, Algorithm \ref{algorithm} is the $Q^\pi(\lambda)$ algorithm \citep{harutyunyan2016q}. The $Q^\pi(\lambda)$ algorithm overcomes the high variance issue in Vanilla IS by introducing the parameter $\lambda$. However, the algorithm converges only when $\lambda$ is sufficiently small \citep{munos2016safe}.
\item
\textit{The TB$(\lambda)$ Algorithm.} When $c(s,a)=\lambda \pi(a|s)$ and $\rho(s,a)=\pi(a|s)/\pi_b(a|s)$, we have the TB$(\lambda)$ algorithm \citep{precup2000eligibility}. 
The TB$(\lambda)$ algorithm also overcomes the high variance issue in Vanilla IS and is guaranteed to converge to $Q^\pi$ without needing any strong assumptions. However, as discussed in \citep{munos2016safe}, the TB$(\lambda)$ algorithm lacks sample efficiency as it does not effectively use the multi-step return.
\item
\textit{The Retrace$(\lambda)$ Algorithm.} When $c(s,a)=\lambda\min(1,\pi(a|s)/\pi_b(a|s))$ and $\rho(s,a)=\pi(a|s)/\pi_b(a|s)$, we have the Retrace$(\lambda)$ algorithm, which overcomes the high variance and converges to $Q^\pi$. The convergence rate of Retrace$(\lambda)$ is empirically observed to be better than TB$(\lambda)$ in \citep{munos2016safe}.
\item
\textit{The $Q$-trace Algorithm.} When $c(s,a)=\min(\bar{c},\pi(a|s)/\pi_b(a|s))$ and $\rho(s,a)=\min(\bar{\rho},\pi(a|s)/\pi_b(a|s))$, where $\bar{\rho}\geq \bar{c}$, Algorithm \ref{algorithm} is the $Q$-trace algorithm \citep{chen2021finite}. 
The $Q$-trace algorithm is an analog of the $V$-trace algorithm \citep{espeholt2018impala} in that $Q$-trace estimates the $Q$-function instead of the $V$-function.
The two truncation levels $\bar{c}$ and $\bar{\rho}$ in these algorithms separately control the variance and the asymptotic bias in the algorithm respectively.
Note that due to the truncation level $\bar{\rho}$, the algorithm no longer converges to $Q^\pi$, but to a biased limit point, denoted by $Q^{\pi,\rho}$. This will be elaborated in detail in Section \ref{sec:variants_off_policy}.
\end{itemize}

From now on, we focus on studying Algorithm \ref{algorithm}. We make the following assumption about the behavior policy $\pi_b$, which is fairly standard in off-policy TD-learning.

\begin{assumption}\label{as:MC}
The behavior policy $\pi_b$ satisfies $\pi_b(a|s)>0$ for all $(s,a)$. In addition, the Markov chain $\{S_k\}$ induced by the behavior policy $\pi_b$ is irreducible and aperiodic.
\end{assumption}

Irreducibility and aperiodicity together imply that the Markov chain $\{S_k\}$ has a unique stationary distribution, which we denote by $\kappa_{S}\in\Delta^{|\mathcal{S}|}$. Moreover, the Markov chain $\{S_k\}$ mixes geometrically fast in that there exist $C>0$ and $\sigma\in (0,1)$ such that $\max_{s\in\mathcal{S}}\|P^k(s,\cdot)-\kappa_S(\cdot)\|_{\text{TV}}\leq C\sigma^k$ for all $k\geq 0$, where $\|\cdot\|_{\text{TV}}$ is the total variation distance \citep{levin2017markov}. Let $\kappa_{SA}\in\Delta^{|\mathcal{S}||\mathcal{A}|}$ be such that $\kappa_{SA}(s,a)=\kappa_S(s)\pi_b(a|s)$ for all $(s,a)$. Note that $\kappa_{SA}\in\Delta^{|\mathcal{S}||\mathcal{A}|}$ is the stationary distribution of the Markov chain $\{(S_k,A_k)\}$ under the behavior policy $\pi_b$. Let $\mathcal{K}_S=\text{diag}(\kappa_S)\in\mathbb{R}^{|\mathcal{S}|\times|\mathcal{S}|}$, and let $\mathcal{K}_{SA}=\text{diag}(\kappa_{SA})\in\mathbb{R}^{|\mathcal{S}||\mathcal{A}|\times|\mathcal{S}||\mathcal{A}|}$. Denote the minimal (maximal) diagonal entries of $\mathcal{K}_S$ and $\mathcal{K}_{SA}$ by $\mathcal{K}_{S,\min}$ ($\mathcal{K}_{S,\max}$) and $\mathcal{K}_{SA,\min}$ ($\mathcal{K}_{S,\max}$) respectively.

\subsection{Identifying the Generalized Bellman Operator}\label{subsec:Bellman}
In this section, we identify the generalized Bellman equation which Algorithm \ref{algorithm} is trying to solve, and also the corresponding generalized Bellman operator and its asynchronous variant. Let $\mathcal{T}_c,\mathcal{H}_\rho:\mathbb{R}^{|\mathcal{S}||\mathcal{A}|}\mapsto\mathbb{R}^{|\mathcal{S}||\mathcal{A}|}$ be two operators defined by
\begin{align*}
    [\mathcal{T}_c(Q)](s,a)&=\sum_{i=0}^{n-1}\gamma^i\mathbb{E}_{\pi_b}\left[\prod_{j=1}^ic(S_j,A_j)Q(S_i,A_i)\;\middle|\; S_0=s,A_0=a\right],\quad \text{and}\\
    [\mathcal{H}_\rho(Q)](s,a)&=\mathcal{R}(s,a)+\gamma\mathbb{E}_{\pi_b}[\rho(S_{k+1},A_{k+1})Q(S_{k+1},A_{k+1})\mid S_k=s,A_k=a]
\end{align*}
for all $(s,a)$. Note that the operator $\mathcal{T}_c(\cdot)$ depends on the generalized importance sampling ratio $c(\cdot,\cdot)$, while the operator $\mathcal{H}_\rho(\cdot)$ depends on the generalized importance sampling ratio $\rho(\cdot,\cdot)$.

With $\mathcal{T}_c(\cdot)$ and $\mathcal{H}_\rho(\cdot)$ defined above, Algorithm \ref{algorithm} can be viewed as an asynchronous SA algorithm for solving the generalized Bellman equation $\mathcal{B}_{c,\rho}(Q)=Q$, where the generalized Bellman operator $\mathcal{B}_{c,\rho}(\cdot)$ is defined by
\begin{align*}
    \mathcal{B}_{c,\rho}(Q)=\mathcal{T}_c(\mathcal{H}_\rho(Q)-Q)+Q.
\end{align*} 
Since Algorithm \ref{algorithm} performs asynchronous update, using the terminology in \citep{chen2021lyapunov}, we further define the asynchronous variant $\Tilde{\mathcal{B}}_{c,\rho}(\cdot)$ of the generalized Bellman operator $\mathcal{B}_{c,\rho}(\cdot)$ by 
\begin{align}\label{def:ABO}
    \Tilde{\mathcal{B}}_{c,\rho}(Q):=\mathcal{K}_{SA}\mathcal{B}_{c,\rho}(Q)+(I-\mathcal{K}_{SA})Q=\mathcal{K}_{SA}\mathcal{T}_c(\mathcal{H}_\rho(Q)-Q)+Q.
\end{align}
Each component of the asynchronous generalized Bellman operator $\Tilde{\mathcal{B}}_{c,\rho}(\cdot)$ can be thought of as a convex combination with identity, where the weights are the stationary probabilities of visiting state-action pairs. This captures the fact that when performing asynchronous update, the corresponding component is updated only when the state-action pair is visited. It is clear from its definition that $\Tilde{\mathcal{B}}_{c,\rho}(\cdot)$ has the same fixed-points as $\mathcal{B}_{c,\rho}(\cdot)$ (provided that they exist). See \citep{chen2021lyapunov} for a more detailed explanation about asynchronous Bellman operators.

Under some mild conditions on the generalized importance sampling ratios $c(\cdot,\cdot)$ and $\rho(\cdot,\cdot)$, we will show in the next section that both the asynchronous generalized Bellman operator $\Tilde{\mathcal{B}}_{c,\rho}(\cdot)$ and the operator $\mathcal{H}_\rho(\cdot)$ are contraction mappings. Therefore, since $\mathcal{T}_c(\bm{0})=\bm{0}$, the operators $\mathcal{H}_\rho(\cdot)$, $\mathcal{B}_{c,\rho}(\cdot)$, $\Tilde{\mathcal{B}}_{c,\rho}(\cdot)$ all share the same unique fixed-point. Since the fixed-point of the operator $\mathcal{H}_\rho(\cdot)$ depends only on the generalized importance sampling ratio $\rho(\cdot,\cdot)$, but not on $c(\cdot,\cdot)$, we can flexibly choose $c(\cdot,\cdot)$ to control the variance while maintaining the fixed-point of the operator $\Tilde{\mathcal{B}}_{c,\rho}(\cdot)$. As we will see later, this is the key property used in designing variants of variance reduced $n$-step off-policy TD-learning algorithms such as $Q^\pi(\lambda)$, TB$(\lambda)$, and Retrace$(\lambda)$.

\subsection{Establishing the Contraction Property}\label{subsec:contraction}
In this section, we study the fixed-point and the contraction property of the asynchronous generalized Bellman operator $\Tilde{\mathcal{B}}_{c,\rho}(\cdot)$. We begin by introducing some notation. Let $D_c,D_\rho\in\mathbb{R}^{|\mathcal{S}||\mathcal{A}|\times|\mathcal{S}||\mathcal{A}|}$ be two diagonal matrices such that $D_c((s,a),(s,a))=\sum_{a\in\mathcal{A}}\pi_b(a|s)c(s,a)$ and $D_\rho((s,a),(s,a))=\sum_{a\in\mathcal{A}}\pi_b(a|s)\rho(s,a)$ for all $(s,a)$. We denote $D_{c,\min}$ ($D_{c,\max}$) and $D_{\rho,\min}$ ($D_{\rho,\max}$) as the minimal (maximal) diagonal entries of the matrices $D_c$ and $D_\rho$ respectively. 

In view of the definition of $\Tilde{\mathcal{B}}_{c,\rho}(\cdot)$ in Eq. (\ref{def:ABO}), any fixed-point of $\mathcal{H}_\rho(\cdot)$ must also be a fixed-point of $\Tilde{\mathcal{B}}_{c,\rho}(\cdot)$. We first study the fixed point of $\mathcal{H}_\rho(\cdot)$ by establishing its contraction property. 

\begin{proposition}\label{prop:H_rho}
Suppose that $D_{\rho,\max}<1/\gamma$. Then the operator $\mathcal{H}_\rho(\cdot)$ is a contraction mapping with respect to the $\ell_\infty$-norm, with contraction factor $\gamma D_{\rho,\max}$. In this case, the fixed-point $Q^{\pi,\rho}$ of $\mathcal{H}_\rho(\cdot)$ satisfies the following two inequalities:
\begin{enumerate}[(1)]
    \item $\|Q^\pi-Q^{\pi,\rho}\|_\infty\leq \frac{\gamma \max_{s\in\mathcal{S}}\sum_{a\in\mathcal{A}}|\pi(a|s)-\pi_b(a|s)\rho(s,a)|}{(1-\gamma)(1-\gamma D_{\rho,\max})}$,
    \item $\|Q^{\pi,\rho}\|_\infty\leq \frac{1}{1-\gamma D_{\rho,\max}}$.
\end{enumerate}
\end{proposition}

Observe from Proposition \ref{prop:H_rho} (1) that when $\rho(s,a)=\pi(a|s)/\pi_b(a|s)$, which is the case for $Q^\pi(\lambda)$, TB$(\lambda)$, and Retrace$(\lambda)$, the unique fixed-point $Q^{\pi,\rho}$ is exactly the target value function $Q^\pi$. This agrees with the definition of the operator $\mathcal{H}_\rho(\cdot)$ in that it reduces to the regular Bellman operator $\mathcal{H}_\pi(\cdot)$ when $\rho(s,a)=\pi(a|s)/\pi_b(a|s)$ for all $(s,a)$. If $\rho(s,a)\neq \pi(a|s)/\pi_b(a|s)$ for some $(s,a)$, then in general the fixed-point of $\mathcal{H}_\rho(\cdot)$ is different from $Q^\pi$. See Appendix \ref{pf:prop:H_fixed-point} for more details. In that case, Proposition \ref{prop:H_rho} provides an error bound on the difference between the potentially biased limit $Q^{\pi,\rho}$ and $Q^\pi$. Such error bound will be useful for us to study the $Q$-trace algorithm in Section \ref{sec:variants_off_policy}. Proposition \ref{prop:H_rho} (2) can be viewed as an analog to the inequality that $\|Q^\pi\|_\infty\leq 1/(1-\gamma)$ for any policy $\pi$. Since $\mathcal{H}_\rho(\cdot)$ is no longer the Bellman operator $\mathcal{H}_\pi(\cdot)$, the corresponding upper bound on the size of its fixed-point $Q^{\pi,\rho}$ also changes. 

Note that Proposition \ref{prop:H_rho} guarantees the existence and uniqueness of the fixed-point of the operator $\mathcal{H}_\rho(\cdot)$, hence also ensures the existence of fixed-points of the asynchronous generalized Bellman operator $\Tilde{\mathcal{B}}_{c,\rho}(\cdot)$.
To further guarantee the uniqueness of the fixed-point of $\Tilde{\mathcal{B}}_{c,\rho}(\cdot)$, we establish its contraction property. We begin with the following definition.
\begin{definition}
Let $\{\mu_i\}_{1\leq i\leq d}$ be such that $\mu_i>0$ for all $i$. Then for any $x\in\mathbb{R}^d$, the weighted $\ell_p$-norm ($p\in [1,\infty)$) of $x$ with weights $\{\mu_i\}$ is defined by $\|x\|_{\mu,p}=(\sum_i\mu_i|x_i|^p)^{1/p}$. 
\end{definition}

We next establish the contraction property of the
operator $\Tilde{\mathcal{B}}_{c,\rho}(\cdot)$ in the following theorem. Let $\omega=\mathcal{K}_{SA,\min}f(\gamma D_{c,\min})(1-\gamma D_{\rho,\max})$, where the function $f:\mathbb{R}\mapsto\mathbb{R}$ is defined by $f(x)=n$ when $x=1$, and  $f(x)=\frac{1-x^n}{1-x}$ when $x\neq 1$. 

\begin{theorem}\label{thm:contraction}
Suppose $c(s,a)\leq \rho(s,a)$ for all $(s,a)$ and $D_{\rho,\max}<1/\gamma$. Then we have the following results.
\begin{enumerate}[(1)]
    \item For any $\theta\in (0,1)$, there exists a weight vector $\mu\in\Delta^{|\mathcal{S}||\mathcal{A}|}$ satisfying $\mu(s,a)\geq \frac{\omega(1-\theta)}{(1-\theta\omega)|\mathcal{S}||\mathcal{A}|}$ for all $(s,a)$ such that the operator $\Tilde{\mathcal{B}}_{c,\rho}(\cdot)$ is a contraction mapping with respect to $\|\cdot\|_{\mu,p}$ for any $p\in [1,\infty)$, with contraction factor $\gamma_c=(1-\omega)^{1-1/p}(1-\theta \omega)^{1/p}$.
    \item The operator $\Tilde{\mathcal{B}}_{c,\rho}(\cdot)$ is a contraction mapping with respect to $\|\cdot\|_\infty$, with contraction factor $\gamma_c=1-\omega$.
\end{enumerate}
\end{theorem}

Consider Theorem \ref{thm:contraction} (1). Observe that we can further upper bound $\gamma_c=(1-\omega)^{1-1/p}(1-\theta \omega)^{1/p}$ by $1-\theta \omega$, which is independent of $p$ and is the uniform contraction factor we are going to use. Theorem \ref{thm:contraction} (2) can be viewed as an extension of Theorem \ref{thm:contraction} because $\lim_{p\rightarrow \infty}\|x\|_{\mu,p}=\|x\|_\infty$ for any $x\in\mathbb{R}^d$ and weight vector $\mu$, and $\lim_{p\rightarrow\infty}(1-\omega)^{1-1/p}(1-\theta \omega)^{1/p}=1-\omega$.

Theorem \ref{thm:contraction} is the key result for our finite-sample analysis, and we present its proof in the next section.
The weighted $\ell_p$-norm (especially the weighted $\ell_2$-norm) contraction property we established for the operator $\Tilde{\mathcal{B}}_{c,\rho}(\cdot)$ has a far-reaching impact even beyond the finite-sample analysis of tabular RL in this paper. Specifically, recall that the  key property used for establishing the convergence and finite-sample bound of on-policy TD-learning with \textit{linear function approximation} in the seminal work \citep{tsitsiklis1997analysis}
is that the corresponding Bellman operator is a contraction mapping not only with respect to the $\ell_\infty$-norm, but also with respect to a weighted $\ell_2$-norm. We establish the same property in the off-policy setting, and hence lay down the foundation for extending our results to the function approximation setting. This is an immediate future research direction.

\subsection{Proof of Theorem \ref{thm:contraction}}
We begin by explicitly computing the asynchronous generalized Bellman operator $\Tilde{\mathcal{B}}_{c,\rho}(\cdot)$. Let $\pi_c$ and $\pi_\rho$ be two policies defined by $\pi_c(a|s)=\frac{\pi_b(a|s)c(s,a)}{D_c((s,a),(s,a))}$ and $\pi_\rho(a|s)=\frac{\pi_b(a|s)\rho(s,a)}{D_\rho((s,a),(s,a))}$ for all $(s,a)$. Let $R\in\mathbb{R}^{|\mathcal{S}||\mathcal{A}|}$ be the reward vector defined by $R(s,a)=\mathcal{R}(s,a)$ for all $(s,a)$. For any policy $\pi'$, let $P_{\pi'}$ be the transition probability matrix of the Markov chain $\{(S_k,A_k)\}$ under $\pi'$, i.e., $P_{\pi'}((s,a), (s',a'))=P_a(s,s')\pi'(a'| s')$ for all state-action pairs $(s,a)$ and $(s',a')$.

\begin{proposition}\label{prop:compute_A}
The operator $\Tilde{\mathcal{B}}_{c,\rho}(\cdot)$ is explicitly given by $\Tilde{\mathcal{B}}_{c,\rho}(Q)=AQ+b$, where
\begin{align*}
    A=I-\mathcal{K}_{SA}\sum_{i=0}^{n-1}(\gamma P_{\pi_c}D_c)^i(I-\gamma P_{\pi_\rho}D_\rho ),\quad \text{and}\quad b=\mathcal{K}_{SA}\sum_{i=0}^{n-1}(\gamma P_{\pi_c}D_c)^iR.
\end{align*}
\end{proposition}
In light of Proposition \ref{prop:compute_A}, to prove Theorem \ref{thm:contraction}, it is enough to study the matrix $A$. To proceed, we require the following definition.
\begin{definition}
Given $\beta\in [0,1]$, a matrix $M\in\mathbb{R}^{d\times d}$ is called a substochastic matrix with modulus $\beta$ if and only if $M_{ij}\geq 0$ for all $i,j$ and $\sum_{j}M_{ij}\leq 1-\beta$ for all $i$.
\end{definition}
\begin{remark}
Note that for any non-negative matrix $M$, we have $\|M\|_\infty=\max_{i}\sum_jM_{ij}$. Therefore, a matrix $M$ being a substochastic matrix with modulus $\beta$ automatically implies that $\|M\|_\infty\leq 1-\beta$.
\end{remark}
We next show in the following two propositions that (1) the matrix $A$ given in Proposition \ref{prop:compute_A} is a substochastic matrix with modulus $\omega$, and (2) for any substochastic matrix $M$ with a positive modulus, there exist weights $\{\mu_i\}$ such that the induced matrix norm $\|M\|_{\mu,p}$ is strictly less than $1$. These two results together immediately imply Theorem \ref{thm:contraction}.
\begin{proposition}\label{prop:substochastic}
Suppose that $c(s,a)\leq \rho(s,a)$ for all $(s,a)$ and $D_{\rho,\max}<1/\gamma$. Then the matrix $A$ given in Proposition \ref{prop:compute_A} is a substochastic matrix with modulus $\omega$, where $\omega=\mathcal{K}_{SA,\min}f(\gamma D_{c,\min})(1-\gamma D_{\rho,\max})$.
\end{proposition}
The condition $c(s,a)\leq \rho(s,a)$ ensures that the matrix $A$ is non-negative, and the condition $D_{\rho,\max}<1/\gamma$ ensures that the each row of the matrix $A$ sums up to at most $1-\omega$. Together they imply the substochasticity of $A$. The modulus $\omega$ is an important parameter for our finite-sample analysis. In view of Theorem \ref{thm:contraction}, we see that large modulus gives smaller (or better) contraction factor of $\Tilde{\mathcal{B}}_{c,\rho}(\cdot)$.
\begin{proposition}\label{prop:matrix_contraction}
For any substochastic matrix $M\in\mathbb{R}^{d\times d}$ with a positive modulus $\beta\in (0,1)$, for any $\theta\in (0,1)$, there exists a weight vector $\mu\in\Delta^{d}$ satisfying $\mu_i\geq \frac{\beta(1-\theta)}{(1-\theta\beta)d}$ for all $i$ such that $\|M\|_{\mu,p}\leq(1-\beta)^{1-1/p}(1-\theta\beta)^{1/p}$ for any $p\in [1,\infty)$. Furthermore, if $M$ is irreducible \footnote{A non-negative matrix is irreducible if and only if its associated graph is strongly connected \citep{berman1994nonnegative}.}, then we can choose $\theta=1$.
\end{proposition}
The result of Proposition \ref{prop:matrix_contraction} further implies $\|M\|_{\mu,p}\leq 1-\theta \beta$, which is independent of the choice of $p$. This implies that $\Tilde{\mathcal{B}}_{c,\rho}(\cdot)$ is a uniform contraction mapping with respect to $\|\cdot\|_{\mu,p}$ for all $p\geq 1$. In general, for different $p$ and $p'$, an operator being a $\|\cdot\|_p$-norm contraction does not imply being a $\|\cdot\|_{p'}$-norm contraction. The reason that we have such a strong uniform contractive result is that the operator $\Tilde{\mathcal{B}}_{c,\rho}(\cdot)$ has a linear structure, and involves a substochastic matrix.

Note that Proposition \ref{prop:matrix_contraction} introduces the tunable parameter $\theta$. It is clear that large $\theta$ gives better contraction factor of $\Tilde{\mathcal{B}}_{c,\rho}(\cdot)$ but worse lower bound on the entries of the weight vector $\mu$. In general, when $M$ is not irreducible, we cannot hope to choose a weight vector $\mu\in\Delta^d$ with positive components and obtain $\|M\|_{\mu,p}\leq 1-\omega$. To see this, consider the example where $M=(1-\omega)[\bm{0},\bm{0},\cdots,\bm{1}]$, which is clearly a substochastic matrix with modulus $\omega$, but is not an irreducible matrix. For any weight vector $\mu\in\Delta^d$, we have $\|M\|_{\mu,p}=(1-\omega)\max_{x\in\mathbb{R}^d:\|x\|_{\mu,p}=1}|x_d|=(1-\omega)/\mu_d^{1/p}> 1-\omega$. However, by choosing $\mu_d$ close to unity, we can get $\|M\|_{\mu,p}$ arbitrarily close to $1-\omega$. This is analogous to choosing $\theta$ close to one in Proposition \ref{prop:matrix_contraction}. Since Proposition \ref{prop:matrix_contraction} is the major result for proving Theorem \ref{thm:contraction}, we provide its proof sketch in Section \ref{sec:proof}.

\subsection{Finite-Sample Convergence Guarantees}\label{subsec:sa_bound}
In light of Theorem \ref{thm:contraction}, Algorithm \ref{algorithm} is a Markovian SA algorithm for solving a fixed-point equation $\Tilde{\mathcal{B}}_{c,\rho}(Q)=Q$, where the fixed-point operator $\Tilde{\mathcal{B}}_{c,\rho}(\cdot)$ is a contraction mapping. Therefore, to establish the finite-sample bounds, we use a Lyapunov drift argument where we choose $W(Q)=\|Q-Q^{\pi,\rho}\|_{\mu,p}^2$ as the Lyapunov function. This leads to a finite-sample bound on $\mathbb{E}[\|Q_k-Q^{\pi,\rho}\|_{\mu,p}^2]$. However, since $\mu$ is unknown, to make the finite-sample bound independent of $\mu$, we use the lower bound on the components of $\mu$ provided in Theorem \ref{thm:contraction}, and also tune the parameters $p$ and $\theta$ to obtain a finite-sample bound on $\mathbb{E}[\|Q_k-Q^{\pi,\rho}\|_\infty^2]$. The fact that we have a uniform contraction factor $1-\theta \omega$ (cf. Theorem \ref{thm:contraction}) plays an important role in such tuning process.

To present the results, we need to introduce more notation. For any $\delta>0$, define $t_\delta(\mathcal{MC}_S)$ as the mixing time of the Markov chain $\{S_k\}$ (induced by $\pi_b$) with precision $\delta$, i.e., $t_\delta(\mathcal{MC}_S)=\min\{k\geq 0:\max_{s\in\mathcal{S}}\|P^k(s,\cdot)-\kappa_S(\cdot)\|_{\text{TV}}\leq \delta\}$. Under Assumption \ref{as:MC}, we can easily verify that $t_\delta(\mathcal{MC}_S)\leq L(\log(1/\delta)+1)$ for some constant $L>0$, which depends only on $C$ and $\delta$. Let $\tau_{\delta,n}=t_\delta(\mathcal{MC}_S)+n+1$. The parameters $c_1,c_2$ and $c_3$ used in stating the following theorem are numerical constants, and will be explicitly given in the Appendix. For ease of exposition, we here only present the finite-sample bound for using constant stepsize. 

\begin{theorem}\label{thm:main}
Consider $\{Q_k\}$ of Algorithm \ref{algorithm}. Suppose that (1) Assumptions \ref{as:MC} is satisfied, (2) $c(s,a)\leq \rho(s,a)$ for all $(s,a)$ and $D_{\rho,\max}< 1/\gamma$, and (3) the constant stepsize $\alpha$ is chosen such that $\alpha \tau_{\alpha,n}\leq \frac{c_1\omega }{\log(2|\mathcal{S}||\mathcal{A}|/\omega) f(\gamma c_{\max})^2(\gamma \rho_{\max}+1)^2}$. Then we have for all $k\geq \tau_{\alpha,n}$:
\begin{align}\label{eq:bounds}
    \mathbb{E}[\|Q_k-Q^{\pi,\rho}\|_\infty^2]
    \leq \zeta_1\left(1-\frac{\omega\alpha}{2}\right)^{k-\tau_{\alpha,n}}+\zeta_2\frac{f(\gamma c_{\max})^2(\gamma\rho_{\max}+1)^2\log(2|\mathcal{S}||\mathcal{A}|/\omega)}{\omega}\alpha \tau_{\alpha,n},
\end{align}
where $\zeta_1=c_2(\|Q_0-Q^{\pi,\rho}\|_\infty+\|Q_0\|_\infty+1)^2$, and $\zeta_2=c_3(3\|Q^{\pi,\rho}\|_\infty+1)^2$.
\end{theorem}
Theorem \ref{thm:main} enables one to design a wide class of off-policy TD variants with provable finite-sample guarantees by choosing appropriate generalized importance sampling ratios $c(\cdot,\cdot)$ and $\rho(\cdot,\cdot)$, which, as we will see soon, are closely related to the bias-variance trade-offs in Algorithm \ref{algorithm}. The first term on the RHS of Eq. (\ref{eq:bounds}) is usually called the bias in SA literature \citep{bottou2018optimization}, and it goes to zero at a geometric rate. The second term on the RHS of Eq. (\ref{eq:bounds}) stands for the variance in the iterates, and it is a constant proportional to $\alpha \tau_{\alpha,n}$. To see more explicitly the bias-variance trade-off, we derive the sample complexity of Algorithm \ref{algorithm} in the following.
\begin{corollary}\label{co:sc}
For an accuracy $\epsilon>0$, to obtain $\mathbb{E}[\|Q_k-Q^{\pi,\rho}\|_\infty]\leq \epsilon$, the sample complexity is
\begin{align}\label{eq:sc}
    \underbrace{\mathcal{O}\left(\frac{\log^2(1/\epsilon)}{\epsilon^2}\right)}_{T_1}\;\underbrace{\Tilde{\mathcal{O}}\left(\frac{1}{(1-\gamma D_{\rho,\max})^2}\right)}_{T_2}\underbrace{\Tilde{\mathcal{O}}\left(\frac{f(\gamma c_{\max})^2(\gamma \rho_{\max}+1)^2}{ \mathcal{K}_{SA,\min}^2f(\gamma D_{c,\min})^2(1-\gamma D_{\rho,\max})^2}\right)}_{T_3}\Tilde{O}\left(n\right).
\end{align}
\end{corollary}
In Corollary \ref{co:sc}, the $\Tilde{\mathcal{O}}(\epsilon^{-2})$ dependence on the accuracy and  the linear dependence on the parameter $n$ (which is due to performing $n$-step bootstrapping) are both the same as $n$-step TD-learning in the on-policy setting \citep{chen2021lyapunov}. The term $T_2$ arises because we upper bound the constants $\zeta_1$ and $\zeta_2$ in Theorem \ref{thm:main} by $\mathcal{O}\left(1/(1-\gamma D_{\rho,\max})^2\right)$, where we make use of Proposition \ref{prop:H_rho} (2). The impact of performing off-policy sampling is captured by the term $T_3$, which depends on the choice of the generalized importance sampling ratios $c(\cdot,\cdot)$ and $\rho(\cdot,\cdot)$. 

In the numerator of $T_3$, we have $f(\gamma c_{\max})^2(\gamma \rho_{\max}+1)^2$, which is from the second term on the RHS of Eq. (\ref{eq:bounds}), and represents the impact of the variance on the sample complexity. It is clear that smaller $c_{\max}$ and $\rho_{\max}$ lead to smaller variance. As we will see later, this is the reason for the variance reduction of various off-policy TD-learning algorithms in the literature. In the denominator of $T_3$, we have $\mathcal{K}_{SA,\min}^2f(\gamma D_{c,\min})^2(1-\gamma D_{\rho,\max})^2=\omega^2$, which represents the effect of the contraction factor. To see this, recall from Theorem \ref{thm:contraction} that the contraction factor is $1-\theta\omega$. In light of the previous analysis, the bias-variance trade-off in general off-policy multi-step TD-learning algorithm \ref{algorithm} is intuitively of the form $\Tilde{\mathcal{O}}\left(\frac{\text{Variance}}{(1-\text{ Contraction factor})^2}\right)$.

\section{Application to Various Off-Policy TD-Learning Algorithms}\label{sec:variants_off_policy}
In this section, we apply Theorem \ref{thm:main} to various off-policy $n$-step TD-learning algorithms in the literature (e.g. $Q^\pi(\lambda)$, TB$(\lambda)$, Retrace$(\lambda)$, and $Q$-trace). We begin by introducing some notation. Let $\pi_{\max}$ ($\pi_{\min}$) and $\pi_{b,\max}$ ($\pi_{b,\min}$) be the maximal (minimal) entry of the target policy $\pi$ and the behavior policy $\pi_b$ respectively. Let $r_{\max}=\max_{s,a}(\pi(a|s)/\pi_b(a|s))$ ($r_{\min}=\min_{s,a}(\pi(a|s)/\pi_b(a|s))$) be the maximum (minimum) ratio between $\pi$ and $\pi_b$. We will overload the notation of $\zeta_1$ and $\zeta_2$ from Theorem \ref{thm:main}. Note that $Q^{\pi,\rho}=Q^\pi$ in $Q^\pi(\lambda)$, TB$(\lambda)$, and Retrace$(\lambda)$, but $Q^{\pi,\rho}\neq Q^\pi$ in $Q$-trace.

\subsection{Finite-Sample Analysis of Vanilla IS}
Consider the Vanilla IS algorithm, where $c(s,a)=\rho(s,a)=\pi(a|s)/\pi_b(a|s)$ for all $(s,a)$. In this case, we have $c_{\max}=\rho_{\max}=r_{\max}$, $D_c=D_\rho=I$, and $\omega=\mathcal{K}_{SA,\min}(1-\gamma^n)$.
\begin{theorem}\label{co:VanillaIS}
Consider Algorithm \ref{algorithm} with Vanilla IS update. Suppose that Assumption \ref{as:MC} is satisfied and the constant stepsize $\alpha$ is chosen such that  $\alpha \tau_{\alpha,n}\leq \frac{c_1\omega }{\log(2|\mathcal{S}||\mathcal{A}|/\omega) ((\gamma r_{\max})^n+1)^2}$. Then we have for all $k\geq \tau_{\alpha,n}$: 
\begin{align*}
    \mathbb{E}[\|Q_k-Q^{\pi}\|_\infty^2]
\leq \zeta_1\left(1-\frac{\omega\alpha}{2}\right)^{k-\tau_{\alpha,n}}+\zeta_2\frac{((\gamma r_{\max})^n+1)^2\log(2|\mathcal{S}||\mathcal{A}|/\omega)}{\omega}\alpha \tau_{\alpha,n},
\end{align*}
where $\omega=\mathcal{K}_{SA,\min}(1-\gamma^n)$. This leads to a sample complexity of $\Tilde{\mathcal{O}}\left(\frac{\log^2(1/\epsilon)}{\epsilon^2}\frac{n((\gamma r_{\max})^n+1)^2}{\mathcal{K}_{SA,\min}^2(1-\gamma^n )^2(1-\gamma)^2}\right)$.
\end{theorem}
Consider the special case where $\pi=\pi_b$ (i.e., on-policy $n$-step TD). Then the sample complexity bound reduces to $\Tilde{\mathcal{O}}\left(\frac{n \log^2(1/\epsilon)}{\epsilon^2\mathcal{K}_{SA,\min}^2(1-\gamma^n )^2(1-\gamma)^2}\right)$, which is comparable to the results in \citep{chen2021lyapunov}. See Appendix \ref{ap:n-step TD} for more details. In the off-policy setting, note that the factor $((\gamma r_{\max})^n+1)^2$ appears in the sample complexity. When $\gamma r_{\max}>1$ (which can usually happen), the sample complexity bound involves an exponential factor $(\gamma r_{\max})^n$. The reason is that the product of importance sampling ratios are not at all controlled by any means in Vanilla IS. Therefore, the variance can be very large. On the other hand, since the importance sampling ratios are not modified, Vanilla IS effectively uses the full $n$-step return. As a result, the parameter $\omega=\mathcal{K}_{SA,\min}(1-\gamma^n)$  within Vanilla IS is the largest (best) among all the algorithms we study.

\subsubsection{Finite-Sample Analysis of $Q^\pi(\lambda)$}
Consider the $Q^\pi(\lambda)$ algorithm, where $c(s,a)=\lambda$ and $\rho(s,a)=\pi(a|s)/\pi_b(a|s)$ for all $(s,a)$. This implies that $c_{\max}=\lambda$, $\rho_{\max}=r_{\max}$, $D_{c}=\lambda I$, and $D_\rho=I$. 
\begin{theorem}\label{co:Qpi}
Consider Algorithm \ref{algorithm} with $Q^\pi(\lambda)$ update. Suppose that Assumption \ref{as:MC} is satisfied, $\lambda\leq r_{\min}$, and the constant stepsize $\alpha$ is chosen such that $\alpha \tau_{\alpha,n}\leq \frac{c_1\omega }{\log(2|\mathcal{S}||\mathcal{A}|/\omega) f(\gamma \lambda)^2(\gamma r_{\max}+1)^2}$. Then we have for all $k\geq \tau_{\alpha,n}$: 
\begin{align*}
    \mathbb{E}[\|Q_k-Q^{\pi}\|_\infty^2]
\leq \zeta_1\left(1-\frac{\omega\alpha}{2}\right)^{k-\tau_{\alpha,n}}+\zeta_2\frac{f(\gamma \lambda)^2(\gamma r_{\max}+1)^2\log(2|\mathcal{S}||\mathcal{A}|/\omega)}{\omega}\alpha \tau_{\alpha,n},
\end{align*}
where $\omega=\mathcal{K}_{SA,\min} f(\gamma\lambda)(1-\gamma)$. This leads to a sample complexity of $\Tilde{\mathcal{O}}\left(\frac{\log^2(1/\epsilon)}{\epsilon^2}\frac{n(\gamma r_{\max}+1)^2}{ \mathcal{K}_{SA,\min}^2(1-\gamma)^4}\right)$.
\end{theorem}
To see how $Q^\pi(\lambda)$ overcomes the high variance issue in Vanilla IS, observe that since $\gamma\lambda\leq \gamma r_{\min}\leq \gamma<1$, we have $f^2(\gamma\lambda)\leq 1/(1-\gamma\lambda)^2$. Therefore, by replacing $c(s,a)=\pi(a|s)/\pi_b(a|s)$ in Vanilla IS with a properly chosen constant $\lambda$, $Q^\pi(\lambda)$ algorithm successfully avoids an exponential large factor in the sample complexity. However, choosing a small $\lambda$ to control the variance has a side effect on the contraction factor. Intuitively, when $\lambda$ is small, $Q^\pi(\lambda)$ does not effectively use the full $n$-step return. Hence the parameter $\omega=\mathcal{K}_{SA,\min}f(\gamma\lambda)(1-\gamma)$ in $Q^\pi(\lambda)$ is less (worse) than the one in Vanilla IS.

\subsubsection{Finite-Sample Analysis of TB$(\lambda)$}
In the TB$(\lambda)$ algorithm, we have $c(s,a)=\lambda \pi(a|s)$ and $\rho(s,a)=\pi(a|s)/\pi_b(a|s)$. This implies that $c_{\max}=\lambda\pi_{\max}$ and $\rho_{\max}=r_{\max}$. Moreover, we have $D_c(s,a)=\lambda \sum_{a}\pi_b(a|s)\pi(a|s)$ and $D_\rho(s,a)=1$ for all $(s,a)$.
\begin{theorem}\label{co:TB}
Consider Algorithm \ref{algorithm} with TB$(\lambda)$ update. Suppose that Assumption \ref{as:MC} is satisfied, $\lambda\leq 1/\pi_{b,\max}$, and the constant stepsize $\alpha$ is chosen such that $\alpha \tau_{\alpha,n}\leq \frac{c_1\omega }{\log(2|\mathcal{S}||\mathcal{A}|/\omega) f(\gamma \lambda\pi_{\max})^2(\gamma r_{\max}+1)^2}$. Then we have for all $k\geq \tau_{\alpha,n}$: 
\begin{align*}
    \mathbb{E}[\|Q_k-Q^{\pi}\|_\infty^2]
\leq \zeta_1\left(1-\frac{\omega\alpha}{2}\right)^{k-\tau_{\alpha,n}}+\zeta_2\frac{f(\gamma \lambda\pi_{\max})^2(\gamma r_{\max}+1)^2\log(2|\mathcal{S}||\mathcal{A}|/\omega)}{\omega}\alpha \tau_{\alpha,n},
\end{align*}
where $\omega=\mathcal{K}_{SA,\min}f(\gamma D_{c,\min})(1-\gamma)$. This implies a sample complexity of 
\begin{align*}
    \Tilde{\mathcal{O}}\left(\frac{\log^2(1/\epsilon)nf(\gamma \lambda \pi_{\max})^2(\gamma r_{\max}+1)^2}{\epsilon^2 \mathcal{K}_{SA,\min}^2f(\gamma D_{c,\min})^2(1-\gamma)^4}\right).
\end{align*}
\end{theorem}
Suppose we further choose $\lambda<1/(\gamma\pi_{\max})$, the TB$(\lambda)$ algorithm also overcomes the high variance issue in Vanilla IS because $f(\gamma \lambda \pi_{\max})\leq 1/(1-\gamma\lambda\pi_{\max})$, which does not involve any exponential large factor. When compared to $Q^\pi(\lambda)$, an advantage of TB$(\lambda)$ is that the constraint on $\lambda$ is much relaxed. However, the same side effect on the contraction factor is also present here. To see this, since $D_{c,\min}=\lambda \min_{s,a}\sum_{a}\pi_b(a|s)\pi(a|s)\leq  1$, the TB$(\lambda)$ algorithm does not effectively use the full $n$-step return. As a result, the parameter $\omega=\mathcal{K}_{SA,\min}f(\gamma D_{c,\min})(1-\gamma)$ in TB$(\lambda)$ is less (worse) than the one in Vanilla IS.

\subsubsection{Finite-Sample Analysis of Retrace$(\lambda)$}
We now analyze the Retrace$(\lambda)$ algorithm, where $c(s,a)=\lambda \min(1,\pi(a|s)/\pi_b(a|s))$ and $\rho(s,a)=\pi(a|s)/\pi_b(a|s)$ for all $(s,a)$. This implies that $c_{\max}=\lambda$ and $\rho_{\max}=r_{\max}$. In addition, we have for any $(s,a)$ that $D_c(s,a)=\lambda \sum_{a}\min (\pi_b(a|s),\pi(a|s))$ and $D_\rho(s,a)=1$. 
\begin{theorem}\label{co:Retrace}
Consider Algorithm \ref{algorithm} with Retrace$(\lambda)$ update. Suppose that Assumption \ref{as:MC} is satisfied, $\lambda\leq 1$, and the constant stepsize $\alpha$ is chosen such that $\alpha \tau_{\alpha,n}\leq \frac{c_1\omega }{\log(2|\mathcal{S}||\mathcal{A}|/\omega) f(\gamma \lambda)^2(\gamma r_{\max}+1)^2}$. Then we have for all $k\geq \tau_{\alpha,n}$: 
\begin{align*}
    \mathbb{E}[\|Q_k-Q^{\pi}\|_\infty^2]
\leq \zeta_1\left(1-\frac{\omega\alpha}{2}\right)^{k-\tau_{\alpha,n}}+\zeta_2\frac{f(\gamma \lambda)^2(\gamma r_{\max}+1)^2\log(2|\mathcal{S}||\mathcal{A}|/\omega)}{\omega}\alpha \tau_{\alpha,n},
\end{align*}
where $\omega=\mathcal{K}_{SA,\min}f(\gamma D_{c,\min})(1-\gamma)$. This implies a sample complexity of $\Tilde{\mathcal{O}}\left(\frac{\log^2(1/\epsilon)nf(\gamma \lambda)^2(\gamma r_{\max}+1)^2}{ \epsilon^2\mathcal{K}_{SA,\min}^2f(\gamma D_{c,\min})^2(1-\gamma)^4}\right)$.
\end{theorem}
Similarly as in $Q^\pi(\lambda)$ and TB$(\lambda)$, the Retrace$(\lambda)$ algorithm overcomes the high variance issue in Vanilla IS by truncating the importance sampling ratio at $1$, which prevents an exponential large factor in the variance term. In addition, it does not require choosing $\lambda$ to be extremely small as required in $Q^\pi(\lambda)$. As for the compromise in the contraction factor, note that $\min(1,\pi(a|s)/\pi_b(a|s))\geq \pi(a|s)$, which implies that $D_c(s,a)$ (and hence $D_{c,\min}$) is larger in the Retrace$(\lambda)$ algorithm than the TB$(\lambda)$ algorithm. Therefore, Retrace$(\lambda)$ does not truncate the $n$-step return as heavy as TB$(\lambda)$. As a result, the parameter $\omega$ is larger (better) in Retrace$(\lambda)$ than in TB$(\lambda)$.

\subsubsection{Finite-Sample Analysis of $Q$-Trace}
Lastly, we analyze the $Q$-trace algorithm, where we choose $c(s,a)=\min(\bar{c},\pi(a|s)/\pi_b(a|s))$ and $\rho(s,a)=\min(\bar{\rho},\pi(a|s)/\pi_b(a|s))$ for all $(s,a)$. This implies that $c_{\max}=\bar{c}$ and $\rho_{\max}=\bar{\rho}$. Moreover, we have for any $(s,a)$ that $D_c(s,a)=\sum_{a}\min(\bar{c}\pi_b(a|s),\pi(a|s))$ and $D_\rho(s,a)=\sum_{a}\min(\bar{\rho}\pi_b(a|s),\pi(a|s))$. 
\begin{theorem}\label{co:Qtrace}
Consider Algorithm \ref{algorithm} with $Q$-trace update. Suppose that Assumption \ref{as:MC} is satisfied, $\bar{c}\leq \bar{\rho}$, and the constant stepsize $\alpha$ is chosen such that $\alpha \tau_{\alpha,n}\leq \frac{c_1\omega }{\log(2|\mathcal{S}||\mathcal{A}|/\omega) f(\gamma\bar{c})^2(\gamma \bar{\rho}+1)^2}$. Then we have for all $k\geq \tau_{\alpha,n}$: 
\begin{align*}
    \mathbb{E}[\|Q_k-Q^{\pi,\rho}\|_\infty^2]
\leq \zeta_1\left(1-\frac{\omega\alpha}{2}\right)^{k-\tau_{\alpha,n}}+\zeta_2\frac{f(\gamma \bar{c})^2(\gamma \bar{\rho}+1)^2\log(2|\mathcal{S}||\mathcal{A}|/\omega)}{\omega}\alpha \tau_{\alpha,n},
\end{align*}
where $\omega=\mathcal{K}_{SA,\min}f(\gamma D_{c,\min})(1-\gamma D_{\rho,\max})$. This implies a sample complexity of 
\begin{align*}
    \Tilde{\mathcal{O}}\left(\frac{\log^2(1/\epsilon)nf(\gamma \bar{c})^2(\gamma \bar{\rho}+1)^2}{ \epsilon^2\mathcal{K}_{SA,\min}^2f(\gamma D_{c,\min})^2(1-\gamma D_{\rho,\max})^2(1-\gamma)^2}\right).
\end{align*}
\end{theorem}
To avoid an exponential large variance, in view of the term $f(\gamma\bar{c})$ in our bound, we need to choose $\bar{c}\leq 1/\gamma$. The major difference between $Q$-trace and Retrace$(\lambda)$ is that the importance sampling ratio $\rho(\cdot,\cdot)$ inside the temporal difference (line 4 of Algorithm \ref{algorithm}) also involves a truncation. As a result, by choosing $\bar{c}=1$ and $\bar{\rho}\geq 1$, the $Q$-trace algorithm enjoys a larger (or better) parameter $\omega$ compared to Retrace$(\lambda)$. However, as shown in Section \ref{subsec:contraction}, due to introducing the truncation level $\bar{\rho}$, the algorithm converges to a biased limit $Q^{\pi,\rho}$ instead of $Q^\pi$. Such truncation bias can be controlled using Proposition \ref{prop:H_rho}. These observations agree with the results \citep{chen2021finite}, where the finite-sample bounds of $Q$-trace were first established.

Compared to \citep{chen2021finite}, we have an improved sample complexity. Specifically, the result in \citep{chen2021finite} implies a sample complexity of $\Tilde{\mathcal{O}}(\frac{\log^2(1/\epsilon)nf(\gamma \bar{c})^2(\gamma \bar{\rho}+1)^2}{\epsilon^2\mathcal{K}_{SA,\min}^3f(\gamma D_{c,\min})^3(1-\gamma D_{\rho,\max})^3(1-\gamma)^2})$, which has an additional factor of $(\mathcal{K}_{SA,\min}f(\gamma D_{c,\min})(1-\gamma D_{\rho,\max}))^{-1}$. Since $\mathcal{K}_{SA,\min}^{-1}\geq |\mathcal{S}||\mathcal{A}|$, our result improves the dependency on the size of the state-action space by a factor of at least $|\mathcal{S}||\mathcal{A}|$ compared to \citep{chen2021finite}. Similarly, since the $V$-trace algorithm \citep{espeholt2018impala} is an analog of the $Q$-trace algorithm, we can also improve the sample complexity for $V$-trace in \citep{chen2021lyapunov}.

\section{Proof sketch of Proposition \ref{prop:matrix_contraction}}\label{sec:proof}
The idea is to construct a stochastic matrix $M''$ such that the following two conditions are satisfied: (1) $M''$ dominates $M$ in the sense that $M''_{ij}\geq M_{ij}$ for all $i,j$, and (2) the Markov chain associated with $M''$ is irreducible, hence admits a unique stationary distribution $\mu$ satisfying $\mu_i>0$ for all $i$. Using $\mu$ as weights, we have the desired result. The detailed analysis is presented in Appendix \ref{pf:prop:matrix_contraction}. We here only present how to construct such a stochastic matrix $M''$.

First of all, consider the special case where $M$ itself is irreducible. Then we first scale up $M$ by a factor of $1/(1-\omega)$ to obtain $M'=\frac{M}{1-\omega}$, which is clearly a substochastic matrix, with modulus zero. Hence there exists a stochastic matrix $M''$ that dominates $M'$ (and also dominates $M$). Moreover, since $M''$ is also irreducible, its associated Markov chain admits a unique stationary distribution $\mu$. This is equivalent to choosing $\theta=1$ in Proposition \ref{prop:matrix_contraction}. In fact, the matrix $M$ being irreducible is only a sufficient condition for us to choose $\theta=1$. What we truly need is the existence of a strictly positive stationary distribution of the stochastic matrix $M''$, which is guaranteed when $M''$ does not have transient states.

Now consider the general case where $M$ is not necessarily irreducible. We construct the intermediate matrix $M'$ by performing a convex combination of the matrix $\frac{M}{1-\omega}$ and the uniform stochastic matrix $\frac{E}{d}$, where $E$ is the all one matrix, with weight $\frac{1-\omega}{1-\theta \omega}$. Specifically, for any $\theta\in (0,1)$, we define
\begin{align*}
    M'=\left(\frac{1-\omega}{1-\theta \omega}\right)\frac{M}{1-\omega}+\left(1-\frac{1-\omega}{1-\theta \omega}\right)\frac{E}{d}.
\end{align*}
Note that $M'$ is a non-negative matrix. In addition, since $M'\bm{1}\leq \frac{1-\omega}{1-\theta \omega}\bm{1}+\left(1-\frac{1-\omega}{1-\theta \omega}\right)\bm{1}=\bm{1}$, where $\bm{1}$ is the all one vector,
the matrix $M'$ is a substochatic matrix with modulus zero, and is also irreducible because all its entries are strictly positive. Therefore, there exists a stochastic matrix $M''$ such that $M''\geq M'$. In addition, since $M''$ also has strictly positive entries, the Markov chain associated with $M''$ is irreducible, hence admits a unique stationary distribution $\mu\in\Delta^d$. By our construction, we can show a lower bound on the components of the stationary distribution $\mu$. 

\section{Conclusion}\label{sec:conclusion}
In this work, we establish finite-sample guarantees of general $n$-step off-policy TD-learning algorithms. The key in our approach is to identify a generalized Bellman operator and establish its contraction property with respect to a weighted $\ell_p$-norm for each $p\in [1,\infty)$, with a uniform contraction factor. Our results are used to derive finite-sample guarantees of variants of $n$-step off-policy TD-learning algorithms in the literature. Specifically, for $Q^\pi(\lambda)$, TB$(\lambda)$, and Retrace$(\lambda)$, we provide the first-known results, and for $Q$-trace, we improve the result in \citep{chen2021finite}. The finite-sample bounds we establish also provide insights about the trade-offs between the bias and the variance. 

\section*{Acknowledgement}
This work was partially supported by ONR Grant N00014-19-1-2566, NSF Grants 1910112, 2019844, NSF Grant CCF-1740776, and an award from Raytheon Technologies. Maguluri  acknowledges seed funding from Georgia Institute of Technology. 

\bibliographystyle{apalike}
\bibliography{references}

\appendix
\appendixpage
\section{Technical Details in Section \ref{sec:RL}}
\subsection{Proof of Proposition \ref{prop:H_rho}}
For any $Q_1, Q_2\in\mathbb{R}^{|\mathcal{S}||\mathcal{A}|}$, and state-action pairs $(s,a)$, using the definition of $\mathcal{H}_\rho(\cdot)$ and we have
\begin{align*}
    &|[\mathcal{H}_\rho(Q_1)](s,a)-[\mathcal{H}_\rho(Q_2)](s,a)|\\
    =\;&\gamma \left|\sum_{s'\in\mathcal{A}}P_a(s,s')\sum_{a'\in\mathcal{A}}\pi_b(a'|s')\rho(s',a')(Q_1(s',a')-Q_2(s',a'))\right|\\
    \leq \;&\gamma \sum_{s'\in\mathcal{A}}P_a(s,s')\sum_{a'\in\mathcal{A}}\pi_b(a'|s')\rho(s',a')|Q_1(s',a')-Q_2(s',a')|\\
    \leq \;&\gamma \|Q_1-Q_2\|_\infty \sum_{s'\in\mathcal{A}}P_a(s,s')\sum_{a'\in\mathcal{A}}\pi_b(a'|s')\rho(s',a')\\
    \leq  \;&\gamma \sum_{s'\in\mathcal{A}}P_a(s,s')D_{\rho,\max}\|Q_1-Q_2\|_\infty\\
    =\;&\gamma D_{\rho,\max}\|Q_1-Q_2\|_\infty.
\end{align*}
It follows that $\|\mathcal{H}_\rho(Q_1)-\mathcal{H}_\rho(Q_2)\|_\infty\leq \gamma D_{\rho,\max}\|Q_1-Q_2\|_\infty$. Since $D_{\rho,\max}<1/\gamma$, the operator $\mathcal{H}_\rho(\cdot)$ is a contraction mapping with respect to $\|\cdot\|_\infty$, with contraction factor $\gamma D_{\rho,\max}$.
\begin{enumerate}[(1)]
    \item We now derive the upper bound on $\|Q^\pi-Q^{\pi,\rho}\|_\infty$. Since $Q^\pi=\mathcal{H}_\pi(Q^\pi)$ and $Q^{\pi,\rho}=\mathcal{H}_\rho(Q^{\pi,\rho}
)$, we have
\begin{align*}
    &|Q^\pi(s,a)-Q^{\pi,\rho}(s,a)|\\
    =\;&|[\mathcal{H}_\pi(Q^\pi)](s,a)-[\mathcal{H}_\rho(Q^{\pi,\rho})](s,a)|\\
    =\;&|[\mathcal{H}_\pi(Q^\pi)](s,a)-[\mathcal{H}_\rho(Q^{\pi})](s,a)+[\mathcal{H}_\rho(Q^{\pi})](s,a)-[\mathcal{H}_\rho(Q^{\pi,\rho})](s,a)|\\
    \leq \;&|[\mathcal{H}_\pi(Q^\pi)](s,a)-[\mathcal{H}_\rho(Q^{\pi})](s,a)|+|[\mathcal{H}_\rho(Q^{\pi})](s,a)-[\mathcal{H}_\rho(Q^{\pi,\rho})](s,a)|\\
    =\;&\gamma \left|\sum_{s'\in\mathcal{S}}P_a(s,s')\sum_{a'\in\mathcal{A}}\left(\pi(a'|s')-\pi_b(a'|s')\rho(s',a')\right)Q^\pi(s',a')\right|+\gamma D_{\rho,\max}\|Q^\pi-Q^{\pi,\rho}\|_\infty\\
    \leq \;&\frac{\gamma}{1-\gamma} \sum_{s'\in\mathcal{S}}P_a(s,s')\sum_{a'\in\mathcal{A}}|\pi(a'|s')-\pi_b(a'|s')\rho(s',a')|+\gamma D_{\rho,\max}\|Q^\pi-Q^{\pi,\rho}\|_\infty\tag{$*$}\\
    \leq \;&\frac{\gamma}{1-\gamma} \max_{s\in\mathcal{S}}\sum_{a\in\mathcal{A}}|\pi(a|s)-\pi_b(a|s)\rho(s,a)|+\gamma D_{\rho,\max}\|Q^\pi-Q^{\pi,\rho}\|_\infty,
\end{align*}
where in Eq. ($*$) we used the inequality $|Q^\pi(s,a)|\leq \sum_{k=0}^\infty\gamma^k=\frac{1}{1-\gamma}$ for all $(s,a)$. Therefore, we have
\begin{align*}
    \|Q^\pi-Q^{\pi_\rho}\|_\infty\leq \frac{\gamma}{1-\gamma} \max_{s\in\mathcal{S}}\sum_{a\in\mathcal{A}}|\pi(a|s)-\pi_b(a|s)\rho(s,a)|+\gamma D_{\rho,\max}\|Q^\pi-Q^{\pi,\rho}\|_\infty.
\end{align*}
Rearranging terms and we obtain the desired result.
\item To prove the upper bound on $\|Q^{\pi,\rho}\|_\infty$, we begin with the fixed-point equation
\begin{align}\label{eq:fp}
    Q^{\pi,\rho}=\mathcal{H}_\rho(Q^{\pi,\rho})=R+\gamma P_{\pi_\rho} D_\rho Q^{\pi,\rho},
\end{align}
where we recall the definition of $D_\rho$ and $\pi_\rho$ in Section \ref{sec:RL}. Eq. (\ref{eq:fp}) is equivalent to $Q^{\pi,\rho}=(I-\gamma P_{\pi_\rho}D_{\rho})^{-1}R$. Therefore, we have
\begin{align*}
    \|Q^{\pi,\rho}\|_\infty=\|(I-\gamma P_{\pi_\rho}D_{\rho})^{-1}R\|_\infty
    \leq \|(I-\gamma P_{\pi_\rho} D_{\rho})^{-1}\|_\infty\|R\|_\infty
    \leq \frac{1}{1-\gamma D_{\rho,\max}}.
\end{align*}
\end{enumerate}

\subsection{On the Fixed-Point of the Operator $\mathcal{H}_\rho(\cdot)$}\label{pf:prop:H_fixed-point}

Suppose that for the state-action pair $(s_0,a_0)$, we have $\rho(s_0,a_0)\neq \pi(a_0|s_0)/\pi_b(a_0|s_0)$. Let an MDP be that the transition probability matrix is an identity matrix for each action, and the reward is zero for all state-action pairs except at $(s_0,a_0)$, where it is equal to $1$. 

In this case, it is clear that for any policy $\pi$, we have $Q^\pi(s,a)=0$ for all $(s,a)\neq (s_0,a_0)$, and $Q^\pi(s_0,a_0)=\frac{1}{1-\gamma}$. Suppose that $Q^\pi=Q^{\pi,\rho}$. then we have
\begin{align*}
    0=\;&Q^\pi(s_0,a_0)-Q^{\pi,\rho}(s_0,a_0)\\
    =\;&\gamma\sum_{s'\in\mathcal{S}}P_{a_0}(s_0,s')\sum_{a'\in\mathcal{A}}\left(\pi(a'|s')-\pi_b(a'|s')\rho(s',a')\right)Q^\pi(s',a')\\
    =\;&\left(\pi(a_0|s_0)-\pi_b(a_0|s_0)\rho(s_0,a_0)\right)Q^\pi(s_0,a_0)\\
    =\;&\frac{1}{1-\gamma}\left(\pi(a_0|s_0)-\pi_b(a_0|s_0)\rho(s_0,a_0)\right).
\end{align*}
This contradicts to the fact that $\rho(s_0,a_0)\neq \pi(a_0|s_0)/\pi_b(a_0|s_0)$. Therefore, we have $Q^\pi\neq Q^{\pi,\rho}$.

\subsection{Proof of Proposition \ref{prop:compute_A}}\label{pf:prop:compute_A}
Recall the definition of $\Tilde{\mathcal{B}}_{c,\rho}(\cdot)$ in Eq. (\ref{def:ABO}):
    \begin{align*}
        \Tilde{\mathcal{B}}_{c,\rho}(Q)=\mathcal{K}_{SA}(\mathcal{B}_{c,\rho}(Q)-Q)+Q=\mathcal{K}_{SA}\mathcal{T}_c(\mathcal{H}_\rho(Q)-Q)+Q.
    \end{align*}
    We first explicitly compute the operators $\mathcal{T}_c(\cdot)$ and $\mathcal{H}_\rho(\cdot)$. For the operator $\mathcal{H}_\rho(\cdot)$, we have from its definition that
    \begin{align*}
        [\mathcal{H}_\rho(Q)](s,a)=\;&\mathcal{R}(s,a)+\gamma\mathbb{E}_{\pi_b}[\rho(S_{k+1},A_{k+1})Q(S_{k+1},A_{k+1})\mid S_k=s,A_k=a]\\
        =\;&\mathcal{R}(s,a)+\gamma\sum_{s'}P_a(s,s')\sum_{a'}\pi_b(a'|s')\rho(s',a')Q(s',a')\\
        =\;&\mathcal{R}(s,a)+\gamma\sum_{s'}P_a(s,s')\sum_{a'}\frac{\pi_b(a'|s')\rho(s',a')}{D_\rho(s',a')} D_\rho(s',a')Q(s',a')\\
        =\;&\mathcal{R}(s,a)+\gamma\sum_{s',a'}P_a(s,s')\pi_\rho(a'|s')D_\rho(s',a')Q(s',a')\\
        =\;&[R+ P_{\pi_\rho} D_\rho Q](s,a).
    \end{align*}
    Note that $P_{\pi_\rho}\in\mathbb{R}^{|\mathcal{S}||\mathcal{A}|\times |\mathcal{S}||\mathcal{A}|}$ here is the transition probability matrix of the Markov chain $\{(S_k,A_k)\}$ under $\pi_\rho$, i.e., $P_{\pi_\rho}((s,a),(s',a'))=P_a(s,s')\pi_\rho(a'|s')$ for any $(s,a)$ and $(s',a')$.
    Hence we have
    \begin{align*}
        \mathcal{H}_\rho(Q)=R+\phantomsection P_{\pi_\rho} D_\rho Q.
    \end{align*}
    
    As for the operator $\mathcal{T}_c(\cdot)$, similarly using the Markov property and the tower property of conditional expectation, we have $\mathcal{T}_c(Q)=\sum_{i=0}^{n-1}(\gamma P_{\pi_c}D_c)^i Q$. It follows that
    \begin{align*}
    \Tilde{\mathcal{B}}_{c,\rho}(Q)&=\mathcal{K}_{SA}\mathcal{T}_c(\mathcal{H}_\rho(Q)-Q)+Q\\
    &=\mathcal{K}_{SA}\sum_{i=0}^{n-1}(\gamma P_{\pi_c}D_c)^i(R+\gamma P_{\pi_{\rho}}D_\rho Q-Q)+Q\\
    &=\underbrace{\left[I-\mathcal{K}_{SA}\sum_{i=0}^{n-1}(\gamma P_{\pi_c}D_c)^i(I-\gamma P_{\pi_\rho}D_\rho )\right]}_{A} Q+\underbrace{\mathcal{K}_{SA}\sum_{i=0}^{n-1}(\gamma P_{\pi_c}D_c)^iR}_{b}.
\end{align*}

\subsection{Proof of Proposition \ref{prop:substochastic}}
Consider the matrix $A$ given in Proposition \ref{prop:compute_A}. To show that $A$ is a substochastic matrix with a positive modulus, we first show that $A$ is non-negative. Observe that
\begin{align}
    A&=I-\mathcal{K}_{SA}\sum_{i=0}^{n-1}(\gamma P_{\pi_c}D_c)^i+\mathcal{K}_{SA}\sum_{i=0}^{n-1}(\gamma P_{\pi_c}D_c)^i\gamma P_{\pi_\rho}D_\rho \nonumber\\
    &=(I-\mathcal{K}_{SA})-\mathcal{K}_{SA}\sum_{i=1}^{n-1}(\gamma P_{\pi_c}D_c)^i+\mathcal{K}_{SA}\sum_{i=0}^{n-1}(\gamma P_{\pi_c}D_c)^i\gamma P_{\pi_\rho}D_\rho\nonumber\\
    &=(I-\mathcal{K}_{SA})-\mathcal{K}_{SA}\sum_{i=0}^{n-2}(\gamma P_{\pi_c}D_c)^{i+1}+\mathcal{K}_{SA}\sum_{i=0}^{n-1}(\gamma P_{\pi_c}D_c)^i\gamma P_{\pi_\rho}D_\rho\nonumber\\
    &=(I-\mathcal{K}_{SA})+\mathcal{K}_{SA}\sum_{i=0}^{n-2}(\gamma P_{\pi_c}D_c)^i\gamma (P_{\pi_\rho}D_\rho-P_{\pi_c}D_c)+\mathcal{K}_{SA}(\gamma P_{\pi_c}D_c)^{n-1}\gamma P_{\pi_\rho}D_\rho.\label{eq:21}
\end{align}
It remains to show that the matrix $P_{\pi_\rho}D_\rho-P_{\pi_c}D_c$ has non-negative entries. For any $(s,a)$ and $(s',a')$, since $c(s',a')\leq \rho(s',a')$ for all $(s',a')$, we have 
\begin{align*}
    [P_{\pi_\rho}D_\rho-P_{\pi_c}D_c]((s,a),(s',a'))&=P_a(s,s')\pi_b(a'|s')(\rho(s',a')-c(s',a'))\geq 0.
\end{align*}

We next show that $A\bm{1}\leq (1-\omega)\bm{1}$, where $\bm{1}\in\mathbb{R}^d$ is the all one vector. Since $A$ is non-negative and $D_{\rho,\max}<1/\gamma$ for all $(s,a)$, we have
\begin{align*}
    \mathcal{K}_{SA}\sum_{i=0}^{n-1}(\gamma P_{\pi_c}D_c)^i(I-\gamma P_{\pi_\rho}D_\rho )\bm{1}&\geq \mathcal{K}_{SA}\sum_{i=0}^{n-1}(\gamma P_{\pi_c}D_c)^i(I-\gamma P_{\pi_\rho}D_{\rho,\max})\bm{1}\\
    &= (1-\gamma D_{\rho,\max})\mathcal{K}_{SA}\sum_{i=0}^{n-1}(\gamma P_{\pi_c}D_c)^i\bm{1}\\
    &\geq \mathcal{K}_{SA,\min}\sum_{i=0}^{n-1}(\gamma D_{c,\min})^i(1-\gamma D_{\rho,\max})\bm{1}\\
    &=\mathcal{K}_{SA,\min}f(\gamma D_{c,\min})(1-\gamma D_{\rho,\max})\bm{1}.
\end{align*}
It follows that
\begin{align*}
    A\bm{1}=\left[I-\mathcal{K}_{SA}\sum_{i=0}^{n-1}(\gamma P_{\pi_c}D_c)^i(I-\gamma P_{\pi_\rho}D_\rho )\right]\bm{1}
    \leq [1-\mathcal{K}_{SA,\min}f(\gamma D_{c,\min})(1-\gamma D_{\rho,\max})]\bm{1}.
\end{align*}
This implies that $A$ is a substochastic matrix with modulus $\omega= \mathcal{K}_{SA,\min}f(\gamma D_{c,\min})(1-\gamma D_{\rho,\max})$.

\subsection{Proof of Proposition \ref{prop:matrix_contraction}}\label{pf:prop:matrix_contraction}
Consider a substochastic matrix $M\in\mathbb{R}^{d\times d}$ with modulus $\beta\in (0,1)$. For any $\theta\in (0,1)$, let
\begin{align*}
    M'=\frac{M}{1-\theta \beta}+\frac{\beta(1-\theta)}{1-\theta\beta}\frac{E}{d},
\end{align*}
where $E$ is the all one matrix. It is clear that $M'> 0$. Moreover, since
\begin{align*}
    M'\bm{1}\leq \frac{1-\beta}{1-\theta \beta}\bm{1}+\frac{\beta(1-\theta)}{1-\theta\beta}\bm{1}=\bm{1},
\end{align*}
the matrix $M'$ is a substochastic matrix with modulus $0$. Therefore, there exists a stochastic matrix $M''$ such that $M''\geq M'>0$. Since $M''$ has strictly positive entries, the Markov chain associated with the stochastic matrix $M''$ is irreducible and aperiodic, hence admits a unique stationary distribution $\mu\in\Delta^d$.  In the special case where $M$ itself is irreducible, we are allowed to choose $\theta=1$ in the preceding construction process, and the resulting stochastic matrix $M''$ is also guaranteed to be irreducible, and hence has a unique stationary distribution $\mu$. 
Since $\mu^\top=\mu^\top M''$, we have
\begin{align*}
    \mu^\top =\mu^\top M''\geq \mu^\top M'\geq \mu^\top \frac{\beta(1-\theta)}{1-\theta\beta}\frac{E}{d}=\frac{\beta(1-\theta)}{(1-\theta\beta)d}\bm{1}.
\end{align*}
This proves the lower bound on the entries of $\mu$.

Now using $\mu$ as the weight vector and we have for any $p\in [1,\infty)$ and $x\in\mathbb{R}^d$:
\begin{align*}
    \|Mx\|_{\mu,p}^p&=\sum_i\mu_i\left|\sum_jM_{ij}x_j\right|^p\\
    &=\sum_i\mu_i\left(\sum_\ell M_{i\ell}\right)^{p}\left|\sum_j\frac{M_{ij}}{\sum_\ell M_{i\ell}} x_j\right|^p\\
    &\leq \sum_i\mu_i\left(\sum_\ell M_{i\ell}\right)^{p-1}\sum_jM_{ij} |x_j|^p\tag{Jensen's inequality}\\
    &\leq (1-\beta)^{p-1}\sum_i\mu_i\sum_jM_{ij} |x_j|^p\\
    &\leq (1-\beta)^{p-1}(1-\theta\beta)\sum_i\mu_i\sum_jM'_{ij} |x_j|^p\tag{definition of $M'$}\\
    &\leq  (1-\beta)^{p-1}(1-\theta\beta)\sum_i\mu_i\sum_jM''_{ij} |x_j|^p\tag{definition of $M''$}\\
    &=(1-\beta)^{p-1}(1-\theta\beta)\sum_j|x_j|^p\sum_i\mu_iM''_{ij} \tag{change of summation order}\\
    &=(1-\beta)^{p-1}(1-\theta\beta)\sum_j\mu_j|x_j|^p\tag{$\mu^\top M''=\mu^\top$}\\
    &=(1-\beta)^{p-1}(1-\theta\beta)\|x\|_{\mu,p}^p.
\end{align*}
It follows that $\|Mx\|_{\mu,p}\leq (1-\omega)^{1-1/p}(1-\theta\beta)^{1/p}\|x\|_{\mu,p}$ for any $x\in\mathbb{R}^d$ and $p\in [1,\infty)$. Using the definition of induced matrix norm immediately gives the result.

\subsection{Proof of Theorem \ref{thm:main}}\label{pf:thm:main}
We first state a more general result in the following, which implies Theorem \ref{thm:main}.

\begin{theorem}\label{thm:more_general}
Consider the iterates $\{Q_k\}$ generated by Algorithm \ref{algorithm}. Suppose that Assumption \ref{as:MC} is satisfied, and $c(s,a)\leq \rho(s,a)$ for all $(s,a)$ and $D_{\rho,\max}<1/\gamma$.
Then for any $\theta\in (0,1)$, there exists a weighted $\ell_p$-norm with weights $\mu\in\Delta^{|\mathcal{S}||\mathcal{A}|}$ satisfying $\mu_{\min}\geq \frac{\omega (1-\theta)}{(1-\theta \omega)|\mathcal{S}||\mathcal{A}|}$ such that the following inequality holds when the constant stepsize $\alpha$ is chosen such that $\alpha \tau_{\alpha,n}\leq \frac{\theta \mu_{\min}^{2/p} \omega}{2052pf(\gamma c_{\max})^2(\gamma \rho_{\max}+1)^2}$:
\begin{align*}
    \mathbb{E}[\|Q_k-Q^{\pi,\rho}\|_{\mu,p}^2]\leq \Tilde{\zeta}_1(1-\theta \omega\alpha)^{k-\tau_{\alpha,n}}+\Tilde{\zeta}_2\frac{pf(\gamma c_{\max})^2(\gamma\rho_{\max}+1)^2}{\mu_{\min}^{2/p}\omega}\alpha \tau_{\alpha,n},
\end{align*}
where $\Tilde{\zeta}_1=(\|Q_0-Q^{\pi,\rho}\|_{\mu,p}+\|Q_0\|_{\mu,p}+1)^2$, and $\Tilde{\zeta}_2=228(3\|Q^{\pi,\rho}\|_{\mu,p}+1)^2$.
\end{theorem}

By using the inequality that $\mu_{\min}^{1/p}\|\cdot\|_p\leq \|\cdot\|_{\mu,p}$ (where $\|\cdot\|_p$ is the unweighted $\ell_p$-norm), Theorem \ref{thm:more_general} implies the following finite-sample bound on $\mathbb{E}[\|Q_k-Q^{\pi,\rho}\|_p]$.
\begin{corollary}\label{co:l_p}
Under same assumptions as Theorem \ref{thm:contraction}, we have for all $k\geq \tau_{\alpha,n}$:
\begin{align*}
    \mathbb{E}[\|Q_k-Q^{\pi,\rho}\|_{p}^2]\leq \frac{\Tilde{\zeta}_1}{\mu_{\min}^{2/p}}(1-\theta \omega\alpha)^{k-\tau_{\alpha,n}}+\frac{\Tilde{\zeta}_2}{\mu_{\min}^{2/p}}\frac{pf(\gamma c_{\max})^2(\gamma\rho_{\max}+1)^2}{\mu_{\min}^{2/p}\omega}\alpha \tau_{\alpha,n},
\end{align*}
\end{corollary}
To proceed and prove Theorem \ref{thm:main}, observe that for any $p\geq 1$ we have
\begin{align*}
    \mathbb{E}[\|Q_k-Q^{\pi,\rho}\|_\infty^2]&\leq  \mathbb{E}[\|Q_k-Q^{\pi,\rho}\|_p^2]\\
    &\leq \frac{\Tilde{\zeta}_1}{\mu_{\min}^{2/p}}(1-\theta \omega\alpha)^{k-\tau_{\alpha,n}}+\frac{\Tilde{\zeta}_2pf(\gamma c_{\max})^2(\gamma\rho_{\max}+1)^2}{\mu_{\min}^{4/p}\omega}\alpha \tau_{\alpha,n}.
\end{align*}
Let $\theta=1/2$ and $p=4\log(1/\mu_{\min})$. Then we have
\begin{align*}
    \frac{1}{\mu_{\min}^{2/p}}&=\mu_{\min}^{-\frac{1}{2\log(1/\mu_{\min})}}=\mu_{\min}^{\frac{1}{2\log(\mu_{\min})}}=\sqrt{e}\leq 2,\quad \text{and}\\
    \frac{p}{\mu_{\min}^{4/p}}&\leq 4e\log(1/\mu_{\min})\leq 4e\log\left(\frac{2|\mathcal{S}||\mathcal{A}|}{\omega}\right)\tag{Using the lower bound on $\mu_{\min}$}.
\end{align*}
It follows that when $\alpha\tau_{\alpha,n}\leq \frac{\omega}{32832\log(2|\mathcal{S}||\mathcal{A}|/\omega)f(\gamma c_{\max})^2(\gamma \rho_{\max}+1)^2}$, we have for all $k\geq \tau_{\alpha,n}$:
\begin{align*}
    \mathbb{E}[\|Q_k-Q^{\pi,\rho}\|_\infty^2]
    &\leq 2\Tilde{\zeta}_1\left(1-\frac{\omega\alpha}{2}\right)^{k-\tau_{\alpha,n}}+4e\Tilde{\zeta}_2\frac{f(\gamma c_{\max})^2(\gamma\rho_{\max}+1)^2\log(2|\mathcal{S}||\mathcal{A}|/\omega)}{\omega}\alpha \tau_{\alpha,n}\\
    &=\zeta_1\left(1-\frac{\omega\alpha}{2}\right)^{k-\tau_{\alpha,n}}+\zeta_2\frac{f(\gamma c_{\max})^2(\gamma\rho_{\max}+1)^2\log(2|\mathcal{S}||\mathcal{A}|/\omega)}{\omega}\alpha \tau_{\alpha,n},
\end{align*}
where in the last line we used $2\Tilde{\zeta}_1\leq \zeta_1=2(\|Q_0-Q^{\pi,\rho}\|_\infty+\|Q_0\|_\infty+1)^2$, and $4e\Tilde{\zeta}_2\leq \zeta_2=912e(3\|Q^{\pi,\rho}\|_\infty+1)^2$. This proves Theorem \ref{thm:main}.

\subsubsection{Proof of Theorem \ref{thm:more_general}}
To prove Theorem \ref{thm:more_general}, we use a Lyapunov drift argument. We next present two approaches for proving Theorem \ref{thm:more_general}. One is by directly using $W(Q)=\frac{1}{2}\|Q\|_{\mu,p}^2$ as the Lyapunov function. Another one is by applying \cite[Theorem 2.1]{chen2021lyapunov}, which studies general stochastic approximation under contraction assumption. 

We begin by rewriting Algorithm \ref{algorithm} using simplified notation. Let $Y_k=(S_k,A_k,\cdots,S_{k+n},A_{k+n})$ for all $k\geq 0$, which is clearly a Markov chain, with finite state-space denoted by $\mathcal{Y}$. Note that under Assumption \ref{as:MC} the Markov chain $\{Y_k\}$ has a unique stationary distribution $\kappa_Y\in\Delta^{|\mathcal{Y}|}$. Define an operator $F:\mathbb{R}^{|\mathcal{S}||\mathcal{A}|}\times \mathcal{Y}\mapsto\mathbb{R}^{|\mathcal{S}||\mathcal{A}|}$ by
\begin{align*}
    &[F(Q,y)](s,a)
    =\;[F(Q,s_0,a_0,...,s_n,a_n)](s,a)\\
    =\;&\mathbb{I}_{\{(s_0,a_0)=(s,a)\}}\sum_{i=0}^{n-1}\gamma^i\prod_{j=1}^ic(s_j,a_j)(\mathcal{R}(s_i,a_i)+\gamma \rho(s_{i+1},a_{i+1})Q(s_{i+1},a_{i+1})-Q(s_i,a_i))+Q(s,a).
\end{align*}
Then the update equation of Algorithm \ref{algorithm} can be equivalently written by $Q_{k+1}=Q_k+\alpha (F(Q_k,Y_k)-Q_k)$.
We next establish in the following proposition the properties of the operators $F(\cdot,\cdot)$ and the Markov chain $\{Y_k\}$, which will be useful in both approaches we present later.

\begin{proposition}\label{prop:properties}
The following statements hold.
\begin{enumerate}[(1)]
\item The operator $F(\cdot)$ satisfies for any $Q_1,Q_2$ and $y$: 
    \begin{enumerate}[(a)]
        \item $\|F(Q_1,y)-F(Q_2,y)\|_{\mu,p}\leq \frac{2}{\mu_{\min}^{1/p}} f(\gamma c_{\max})(\gamma \rho_{\max}+1)\|Q_1-Q_2\|_{\mu,p}$,
        \item $\|F(\bm{0},y)\|_{\mu,p}\leq f(\gamma c_{\max})$.
    \end{enumerate} 
    \item For any $k\geq 0$ and $n\geq 0$, we have $\max_{y\in\mathcal{Y}}\|P^{k+n+1}(y,\cdot)-\kappa_Y(\cdot)\|_{\text{TV}}\leq C\sigma^k$.
    \item For any $Q$, we have $\mathbb{E}_{Y\sim \kappa_Y}[F(Q,Y)]=\Tilde{\mathcal{B}}_{c,\rho}(Q)$.
\end{enumerate}
\end{proposition}

We now present our first approach of proving Theorem \ref{thm:more_general}, where we directly use $W(Q)=\frac{1}{2}\|Q\|_{\mu,p}^2$ as the Lyapunov function.\\

\noindent \textbf{First Approach:}
Note that the function $W(Q)=\frac{1}{2}\|Q\|_{\mu,p}^2$ is a $(p-1)$-smooth function with respect to $\|\cdot\|_{\mu,p}$ \cite{beck2017first}, i.e., $W(Q_2)\leq W(Q_1)+\langle \nabla W(Q_1),Q_2-Q_1\rangle+\frac{p-1}{2}\|Q_1-Q_2\|_{\mu,p}^2 $
for any $Q_1,Q_2\in\mathbb{R}^{|\mathcal{S}||\mathcal{A}|}$. Therefore, using the update equation of Algorithm \ref{algorithm}, we have for any $k\geq 0$:
\begin{align*}
	\mathbb{E}[W(Q_{k+1}-Q^{\pi,\rho})]
	\leq \;&\mathbb{E}[W(Q_k-Q^{\pi,\rho})]+\alpha_k\mathbb{E}[\langle \nabla W(Q_k-Q^{\pi,\rho}),Q_{k+1}-Q_k\rangle]\\
	&+\frac{(p-1)\alpha_k^2}{2}\mathbb{E}[\|Q_{k+1}-Q_k\|_p^2].
\end{align*}
The rest of the proof is identical to that of \cite[Theorem 2.1]{chen2021lyapunov} (where Proposition \ref{prop:properties} plays an important role), and is omitted.
Here, we can directly use  $W(Q)$ as a Lyapunov function because it is smooth. In contrast, 
\cite{chen2020finite,chen2021lyapunov} study the more general settings when it is not smooth. In that case, the Lyapunov function is obtained by using a smoothing technique involving generalized Moreau envelop and infimal convolution to obtain a smooth approximation of  $W(Q)$. 
One can of course, directly apply the result in \cite{chen2020finite,chen2021lyapunov}, which we present as a second approach.\\

\noindent\textbf{Second Approach:} We next present how to apply \cite[Theorem 2.1]{chen2021lyapunov} to obtain the results. We begin by restating Theorem 2.1 of \citep{chen2021lyapunov} in the case of weighted $\ell_p$-norm contraction with weights $\{\mu_i\}_{1\leq i\leq d}$. Using the notation of \cite{chen2021lyapunov}, we choose the smoothing norm $\|\cdot\|_s$ to be the same norm as the contraction norm: $\|\cdot\|_{\mu,p}$.
\begin{theorem}[Theorem 2.1 in \citep{chen2021lyapunov}]\label{thm:chen}
Consider the SA algorithm
\begin{align}\label{eq:15}
    x_{k+1}=x_k+\alpha (F(x_k,Y_k)-x_k).
\end{align}
Suppose that
\begin{enumerate}[(1)]
    \item The random process $\{Y_k\}$ is a Markov chain (denoted by $\mathcal{MC}_Y$) with finite state-space $\mathcal{Y}$. In addition, $\{Y_k\}$ has a unique stationary distribution $\kappa_Y$, and there exist $C_1>0$ and $\sigma_1\in (0,1)$ such that $\max_{y\in\mathcal{Y}}\|P^k(y,\cdot)-\kappa_Y(\cdot)\|_{\text{TV}}\leq C_1\sigma_1^k$ for all $k\geq 0$.
    \item The operator $F:\mathbb{R}^d \times\mathcal{Y}\mapsto\mathbb{R}^d$ satisfies for any $x_1,x_2\in\mathbb{R}^d$ and $y\in\mathcal{Y}$
    \begin{enumerate}[(a)]
        \item $\|F(x_1,y)-F(x_2,y)\|_{\mu,p}\leq a_1\|x_1-x_2\|_{\mu,p}$, where $a_1>0$ is a constant,
        \item $\|F(\bm{0},y)\|_{\mu,p}\leq b_1$, where $b_1>0$ is a constant.
    \end{enumerate}
    \item The expected operator $\bar{F}:\mathbb{R}^d\mapsto\mathbb{R}^d$ defined by $\bar{F}(x)=\mathbb{E}_{Y\sim \kappa_Y}[F(x,Y)]$ satisfies $\bar{F}(x^*)=x^*$, and is a contraction mapping with respect to $\|\cdot\|_{\mu,p}$, with contraction factor $\gamma_c\in (0,1)$.
    \item The constant stepsize $\alpha$ is chosen such that $\alpha t_\alpha(\mathcal{MC}_Y)\leq\frac{1-\gamma_c}{228p(a_1+1)^2}$.
\end{enumerate}
Then we have for all $k\geq t_\alpha(\mathcal{MC}_Y)$ that
\begin{align*}
    \mathbb{E}[\|x_k-x^*\|_{\mu,p}^2]\leq \Tilde{c}_1(1-(1-\gamma_c)\alpha)^{k-t_\alpha(\mathcal{MC}_Y)}+\frac{228p\Tilde{c}_2}{(1-\gamma_c)}\alpha t_\alpha(\mathcal{MC}_Y),
\end{align*}
where $\Tilde{c}_1=(\|x_0-x^*\|_{\mu,p}+\|x_0\|_{\mu,p}+b_1/(a_1+1))^2$ and $\Tilde{c}_2=((a_1+1)\|x^*\|_{\mu,p}+b_1)^2$.
\end{theorem}

Proposition \ref{prop:properties} in conjunction with Theorem \ref{thm:contraction} imply that the requirements for applying Theorem \ref{thm:chen} are satisfied. For any $\theta \in (0,1)$, when the constant stepsize $\alpha$ is chosen such that $\alpha \tau_{\alpha,n}\leq \frac{\theta \mu_{\min}^{2/p} \omega}{2052pf(\gamma c_{\max})^2(\gamma \rho_{\max}+1)^2}$,
we have for any $k\geq \tau_{\alpha,n}$:
\begin{align*}
    \mathbb{E}[\|Q_k-Q^{\pi,\rho}\|_{\mu,p}^2]\leq \Tilde{\zeta}_1(1-\theta \omega\alpha)^{k-\tau_{\alpha,n}}+\Tilde{\zeta}_2\frac{pf(\gamma c_{\max})^2(\gamma\rho_{\max}+1)^2}{\mu_{\min}^{2/p}\omega}\alpha \tau_{\alpha,n},
\end{align*}
where $\Tilde{\zeta}_1=(\|Q_0-Q^{\pi,\rho}\|_{\mu,p}+\|Q_0\|_{\mu,p}+1)^2$, and $\Tilde{\zeta}_2=228(3\|Q^{\pi,\rho}\|_{\mu,p}+1)^2$.

\subsubsection{Proof of Proposition \ref{prop:properties}}

\begin{enumerate}[(1)]
    \item For any $Q_1,Q_2\in\mathbb{R}^{|\mathcal{S}||\mathcal{A}|}$ and $y=(s_0,a_0,\cdots,s_n,a_n)\in\mathcal{Y}$, we have
    \begin{align*}
    &\|F(Q_1,s_0,a_0,...,s_n,a_n)-F(Q_2,s_0,a_0,...,s_n,a_n)\|_{\mu,p}\\
    \leq \;&\left[\sum_{s,a}\mu(s,a)\left(\mathbb{I}_{\{(s,a)=(s_0,a_0)\}}\sum_{i=0}^{n-1}(\gamma c_{\max})^i(\gamma\rho_{\max}+1)\|Q_1-Q_2\|_\infty\right)^p\right]^{1/p}\\
    &+\|Q_1-Q_2\|_{\mu,p}\tag{Triangle inequality}\\
    =\;&f(\gamma c_{\max})(\gamma \rho_{\max}+1)\|Q_1-Q_2\|_\infty+\|Q_1-Q_2\|_{\mu,p}.\\
    \leq \;&\frac{2}{\mu_{\min}^{1/p}} f(\gamma c_{\max})(\gamma \rho_{\max}+1)\|Q_1-Q_2\|_{\mu,p}.
\end{align*}
Similarly, for any $y=(s_0,a_0,\cdots,s_n,a_n)\in\mathcal{Y}$, we have
\begin{align*}
    \|F(\bm{0},s_0,a_0,...,s_n,a_n)\|_{\mu,p}
    \leq \left[\sum_{s,a}\mu(s,a)\mathbb{I}_{\{(s,a)=(s_0,a_0)\}}\left(\sum_{i=0}^{n-1}(\gamma c_{\max})^i\right)^p\right]^{1/p}\leq f(\gamma c_{\max}).
\end{align*}
\item Under Assumption \ref{as:MC}, it is clear that $\{Y_k\}$ has a unique stationary distribution, which we have denoted by $\kappa_Y$, and is given by
\begin{align*}
    \kappa_Y(s_0,a_0,...,s_n,a_n)=\kappa_S(s_0)\left(\prod_{i=0}^{n-1}\pi(a_i|s_i)P_{a_i}(s_i,s_{i+1})\right)\pi(a_n|s_n).
\end{align*}
Now use the definition of total variation distance, and we have for any $y=(s_0,a_0,...,s_n,a_n)$ and $k\geq 0$:
\begin{align*}
    &\|P^{k+n+1}((s_0,a_0,...,s_n,a_n),\cdot)-\kappa_Y(\cdot)\|_{\text{TV}}\\
    =\;&\frac{1}{2}\sum_{s_0',a_0',...,s_n',a_n'}\left|\sum_{s}P_{a_n}(s_n,s)P_{\pi_b}^{k}(s,s_0')-\kappa_S(s_0')\right|\left(\prod_{i=0}^{n-1}\pi(a_i'|s_i')P_{a_i'}(s_i',s_{i+1}')\right)\pi(a_n'|s_n')\\
    =\;&\frac{1}{2}\sum_{s_0'}\left|\sum_{s}P_{a_n}(s_n,s)P_{\pi_b}^{k}(s,s_0')-\kappa_S(s_0')\right|\\
    \leq \;&\frac{1}{2}\sum_{s_0'}\sum_{s}P_{a_n}(s_n,s)\left|P_{\pi_b}^{k}(s,s_0')-\kappa_S(s_0')\right|\\
    =\;&\frac{1}{2}\sum_{s}P_{a_n}(s_n,s)\sum_{s_0'}\left|P_{\pi_b}^{k}(s,s_0')-\kappa_S(s_0')\right|\\
    \leq \;&\frac{1}{2}\sum_{s}P_{a_n}(s_n,s)\max_{s'}\sum_{s_0'}\left|P_{\pi_b}^{k}(s',s_0')-\kappa_S(s_0')\right|\\
    = \;&\max_{s\in\mathcal{S}}\|P_{\pi_b}^k(s,\cdot)-\kappa_S(\cdot)\|_{\text{TV}}\\
    \leq \;& C\sigma^k.
\end{align*}
\item It is clear that $\mathbb{E}_{Y\sim\mathcal{K}_Y}[F(Q,Y)]=\mathcal{K}_{SA}\mathcal{T}_c(\mathcal{H}_\rho(Q)-Q)+Q$, which by definition is equal to $\Tilde{\mathcal{B}}_{c,\rho}(Q)$.
\end{enumerate}

\section{Connection to Linear SA Involving a Hurwitz Matrix}
In view of Proposition \ref{prop:compute_A}, Algorithm \ref{algorithm} can be alternatively interpreted as a linear SA algorithm for solving the equation $(A-I)Q+b=0$. In the case where the matrix $A$ is substochastic, our results imply finite-sample bounds for such linear SA algorithm, which is stated in the following. 

Let $\{Y_k\}$ be a Markov chain with finite state-space $\mathcal{Y}$ and unique stationary distribution $\kappa_Y$. Let $\Tilde{A}:\mathcal{Y}\mapsto\mathbb{R}^{d\times d}$ be a matrix valued function and let $\Tilde{b}:\mathcal{Y}\mapsto\mathbb{R}^d$ be a vector valued function. Let $\bar{A}=\mathbb{E}_{Y\sim \kappa_Y}[\Tilde{A}(Y)]$ and $\bar{b}=\mathbb{E}_{Y\sim \kappa_Y}[\Tilde{b}(Y)]$. Consider the following linear SA algorithm:
\begin{align}\label{algo:linear_sa}
    x_{k+1}=x_k+\alpha((\Tilde{A}(Y_k)-I)x_k+\tilde{b}(Y_k)),
\end{align}
where $\alpha$ is the constant stepsize. Then, we have the following result. 

\begin{theorem}\label{thm:linear_SA}
Consider $\{x_k\}$ generated by Algorithm \eqref{algo:linear_sa}.
Suppose that 
\begin{enumerate}[(1)]
    \item The Markov chain $\{Y_k\}$ has a unique stationary distribution $\kappa_Y$, and $\max_{y\in\mathcal{Y}}\|P^k(y,\cdot)-\kappa_Y(\cdot)\|_{\text{TV}}\leq C'\sigma'^k$ for all $k\geq 0$, where $C'>0$ and $\sigma'\in (0,1)$ are constants.
    \item There exist $A_{\max},b_{\max}>0$ such that $\|\Tilde{A}(y)\|_\infty\leq A_{\max}$ and $\|\Tilde{b}(y)\|_\infty\leq b_{\max}$ for all $y\in\mathcal{Y}$.
    \item The matrix $\bar{A}$ is a sub-stochastic matrix with modulus $\omega'\in (0,1)$.
\end{enumerate}
Then, for any $\theta\in (0,1)$, when $\alpha$ is chosen such that $\alpha t_\alpha(\mathcal{MC}_Y)\leq \frac{\theta \omega'\mu_{\min}^{2/p}}{228 p(A_{\max}+1)^2}$, there exists a weight vector $\mu\in\Delta^d$ satisfying $\mu_{\min}\geq \frac{\omega'(1-\theta)}{(1-\theta \omega')d}$ such that we have for all $k\geq t_\alpha(\mathcal{MC}_Y)$:
\begin{align*}
    \mathbb{E}[\|x_k-x^*\|_{\mu,p}^2]\leq \Tilde{c}_1(1-\theta\omega'\alpha)^{k-t_\alpha(\mathcal{MC}_Y)}+\frac{228p \Tilde{c}_2( A_{\max}+1)^2}{\mu_{\min}^{2/p} \theta \omega'}\alpha t_\alpha(\mathcal{MC}_Y),
\end{align*}
where $\Tilde{c}_1=(\|x_0-x^*\|_{\mu,p}+\|x_0\|_{\mu,p}+\frac{b_{\max}\mu_{\min}^{1/p}}{A_{\max}+1})^2$ and $\Tilde{c}_2=(\|x^*\|_{\mu,p}+\frac{b_{\max}\mu_{\min}^{1/p}}{A_{\max}+1})^2$.
\end{theorem}

Similarly, Theorem \ref{thm:linear_SA} has the following two corollaries, where we provide finite-sample bounds on $\mathbb{E}[\|x_k-x^*\|_p^2]$ and $\mathbb{E}[\|x_k-x^*\|_\infty^2]$.
\begin{corollary}
Under the same assumptions as Theorem \ref{thm:linear_SA}, we have for all $k\geq t_\alpha(\mathcal{MC}_Y)$:
\begin{align*}
    \mathbb{E}[\|x_k-x^*\|_p^2]\leq \frac{\Tilde{c}_1}{\mu_{\min}^{2/p}}(1-\theta \omega'\alpha)^{k-t_\alpha(\mathcal{MC}_Y)}+\frac{228p \Tilde{c}_2( A_{\max}+1)^2}{\mu_{\min}^{4/p} \theta \omega'}\alpha t_\alpha(\mathcal{MC}_Y).
\end{align*}
\end{corollary}

\begin{corollary}
Under the same assumptions as Theorem \ref{thm:linear_SA}, we have for all $k\geq t_\alpha(\mathcal{MC}_Y)$:
\begin{align*}
    \mathbb{E}[\|x_k-x^*\|_\infty^2]\leq \Tilde{c}_1\sqrt{e}\left(1-\frac{\omega'\alpha}{2} \right)^{k-t_\alpha(\mathcal{MC}_Y)}+\frac{1824e\log(2d/\omega')\Tilde{c}_2 (A_{\max}+1)^2 }{\omega'}\alpha t_\alpha(\mathcal{MC}_Y).
\end{align*}
\end{corollary}
An alternative approach of studying linear SA algorithm with Markovian noise was provided in \cite{srikant2019finite}, which established convergence bounds for linear stochastic approximation when the matrix $\bar{A}-I$ is Hurwitz. The bounds in \cite{srikant2019finite} are in terms of solution of the Lyapunov equation:
\begin{align}\label{eq:Lyapunov}
    (\bar{A}-I)^\top\Sigma+\Sigma(\bar{A}-I)+I=0.
\end{align} 
In particular, the finite-sample bounds depend on the ratio between maximum eigenvalue $\lambda_{\max}$ and minimum eigenvalue $\lambda_{\min}$ of the solution $\Sigma$. In general, it is not clear how this ratio can be evaluated and it is unknown. In the context of off-policy TD algorithms, the dependence on the contraction factor, the variance and the sizes of state-action spaces are hidden in this ratio, and so the trade-offs that we presented in Section \ref{sec:RL} are not evident. 

In our approach, we overcome this challenge by finding a $\Sigma$ that solves the  inequality,
\begin{align} \label{eq:Lyapunov_inequality}
    (\bar{A}-I)^\top\Sigma+\Sigma(\bar{A}-I)+\eta\Sigma\preceq \footnote{We write $M_1 \succeq (\preceq)M_2$ to mean that $M_1-M_2$ ($M_2-M_1$) is a positive semi-definite matrix.} \;0 
\end{align}  instead of the Lyapunov equation. Here,  $\eta>0$ is a constant. It turns out that this is sufficient to obtain finite-sample bounds. We find such a $\Sigma$ by essentially establishing the weighted $\ell_p$-norm contraction property, in particular, the weighted $\ell_2$-norm contraction property. To see this, observe that when the matrix $\bar{A}$ is a substochastic matrix with a positive modulus $\omega'\in (0,1)$, Theorem \ref{thm:contraction} implies that $\bar{A}^\top N \bar{A}\preceq (1-\omega') N$, where $N=\text{diag}(\mu)$. Therefore, we have
\begin{center}
    $(1-\omega') N\succeq (\bar{A}-I+I)^\top N(\bar{A}-I+I)
    \succeq (\bar{A}-I)^\top N+N(\bar{A}-I)+N$,
\end{center}
which implies $(\bar{A}-I)^\top N+N(\bar{A}-I)+\omega' N\preceq 0$.
Thus, $\Sigma = N$ satisfies \eqref{eq:Lyapunov_inequality} with  $\eta = \omega'$. When compared to the approach in \cite{srikant2019finite}, we trade-off the unknown eigenvalues of the solution to the Lyapunov equation for the weight vector $\mu$. Since we are able to establish a lower bound on $\mu$, we can obtain a sample complexity bound that doesn't involve any unknowns (except $\mathcal{K}_{SA,\min}$, which is inevitable in both approaches). As a result, we are able to fully characterize the impact of the generalized importance sampling ratios in Corollary \ref{co:sc}, and provide insights about the bias-variance trade-offs in multi-step off-policy TD-learning algorithms.

In this section, we present finite-sample bounds for linear stochastic approximation involving a substochastic matrix, whereas \cite{srikant2019finite} considers Hurwitz matrices. However, these results are equivalent  because of the following lemma. 
\begin{lemma}
\begin{enumerate}[(1)]
    \item Suppose $M\in\mathbb{R}^{d\times d}$ is a substochastic matrix with a positive modulus, then the matrix $M'$ defined by $M'=M-I$ is Hurwitz.
    \item Suppose $M'\in\mathbb{R}^{d\times d}$ is a Hurwitz matrix, then there exists $\phi \in (0,1)$ such that the matrix $M=\phi M'+I$ is a contraction mapping with respect to a norm induced by an inner product.
\end{enumerate}
\end{lemma}
\begin{proof}
The proof of Part (1) is straight-forward and we skip it. Part (2) is also not challenging, and we present an overview of the argument. When $M'$ is Hurwitz, all its eigenvalues are located on the open left half of the complex plane. Therefore, there exists $\phi\in (0,1)$ such that the eigenvalues of $M'$ are within the unit ball centered at $(-1,0)$ of the complex plane. It follows that $M=\phi M'+I$ has eigenvalues located inside the unit ball centered at the origin of the complex plane. This implies the desired contraction property.
\end{proof}

\section{Technical Details in Section \ref{sec:variants_off_policy}}
\subsection{Proof of Theorem \ref{co:VanillaIS}}
Since Vanilla IS is a special case of Algorithm \ref{algorithm}, one can directly apply Theorem  \ref{thm:main} to obtain the finite-sample bound. However, there is one special property of Vanilla IS we can exploit to obtain a tighter finite-sample bound. In particular, consider Proposition \ref{prop:properties} (1) (a). In the case of Vanilla IS, the corresponding Lispchitz constant is $\frac{2}{\mu_{\min}^{1/p}} f(\gamma r_{\max})(\gamma r_{\max}+1)$. We next show that due to $c(s,a)=\rho(s,a)$ in Vanilla IS, we can use telescoping to improve the Lipschitz constant. Specifically, in Vanilla IS, for any $Q\in\mathbb{R}^{|\mathcal{S}||\mathcal{A}|}$, $y\in\mathcal{Y}$, and $(s,a)$, we have
\begin{align*}
    &[F(Q,y)](s,a)\\
    =\;&\mathbb{I}_{\{(s_0,a_0)=(s,a)\}}\sum_{i=0}^{n-1}\gamma^i\prod_{j=1}^ic(s_j,a_j)(\mathcal{R}(s_i,a_i)+\gamma c(s_{i+1},a_{i+1})Q(s_{i+1},a_{i+1})-Q(s_i,a_i))+Q(s,a)\\
    =\;&\mathbb{I}_{\{(s_0,a_0)=(s,a)\}}\sum_{i=0}^{n-1}\gamma^i\prod_{j=1}^ic(s_j,a_j)\mathcal{R}(s_i,a_i)+\mathbb{I}_{\{(s_0,a_0)=(s,a)\}}\sum_{i=0}^{n-1}\gamma^{i+1}\prod_{j=1}^{i+1}c(s_j,a_j)Q(s_{i+1},a_{i+1})\\
    &-\mathbb{I}_{\{(s_0,a_0)=(s,a)\}}\sum_{i=0}^{n-1}\gamma^i\prod_{j=1}^ic(s_j,a_j)Q(s_i,a_i)+Q(s,a)\\
    =\;&\mathbb{I}_{\{(s_0,a_0)=(s,a)\}}\sum_{i=0}^{n-1}\gamma^i\prod_{j=1}^ic(s_j,a_j)\mathcal{R}(s_i,a_i)+\mathbb{I}_{\{(s_0,a_0)=(s,a)\}}\sum_{i=1}^{n}\gamma^{i}\prod_{j=1}^{i}c(s_j,a_j)Q(s_i,a_i)\\
    &-\mathbb{I}_{\{(s_0,a_0)=(s,a)\}}\sum_{i=0}^{n-1}\gamma^i\prod_{j=1}^ic(s_j,a_j)Q(s_i,a_i)+Q(s,a)\\
    =\;&\mathbb{I}_{\{(s_0,a_0)=(s,a)\}}\sum_{i=0}^{n-1}\gamma^i\prod_{j=1}^ic(s_j,a_j)\mathcal{R}(s_i,a_i)+\mathbb{I}_{\{(s_0,a_0)=(s,a)\}}\gamma^n\prod_{j=1}^{n}c(s_j,a_j)Q(s_n,a_n)\\
    &+(1-\mathbb{I}_{\{(s_0,a_0)=(s,a)\}})Q(s,a).
\end{align*}

Therefore, we have for any $Q_1,Q_2\in\mathbb{R}^{|\mathcal{S}||\mathcal{A}|}$, and $y\in\mathcal{Y}$:
\begin{align*}
    &\|F(Q_1,y)-F(Q_2,y)\|_{\mu,p}\\
    \leq\;& \left[\sum_{s,a}\mu(s,a)\left|\mathbb{I}_{\{(s_0,a_0)=(s,a)\}}\gamma^n\prod_{j=1}^{n}c(s_j,a_j)(Q_1(s_n,a_n)-Q_2(s_n,a_n))\right|^p\right]^{1/p}+\|Q_1-Q_2\|_{\mu,p}\\
    \leq\;& \left[\sum_{s,a}\mu(s,a)\left|(\gamma r_{\max})^n\|Q_1-Q_2\|_\infty\right|^p\right]^{1/p}+\|Q_1-Q_2\|_{\mu,p}\\
    \leq \;&(\gamma r_{\max})^n\|Q_1-Q_2\|_\infty+\|Q_1-Q_2\|_{\mu,p}\\
    \leq \;&\frac{(\gamma r_{\max})^n+1}{\mu_{\min}^{1/p}}\|Q_1-Q_2\|_{\mu,p}.
\end{align*}
Using this improved Lipschitz constant and we obtain Theorem \ref{co:VanillaIS}, where the rest of the proof is identical to that of Theorem \ref{thm:main}. 

\subsection{Comparison to the $n$-Step TD-Learning Results in \cite{chen2021lyapunov}}\label{ap:n-step TD}

The sample complexity of on-policy $n$-step TD-learning provided in \cite[Corollary 3.3.1.]{chen2021lyapunov} is
\begin{align}\label{eq:sc:chen}
		\tilde{\mathcal{O}}\left(\frac{n\log^2(1/\epsilon)}{\epsilon^2\mathcal{K}_{S,\min}^2(1-\gamma)^2(1-\gamma^n)^2}\right)\tilde{\mathcal{O}}(|\mathcal{S}|^{1/2}).
\end{align}
In this work, by setting $\pi_b=\pi$, Theorem \ref{co:VanillaIS} implies a sample complexity of
\begin{align}\label{eq:sc:this}
    \Tilde{\mathcal{O}}\left(\frac{n \log^2(1/\epsilon)}{\epsilon^2\mathcal{K}_{SA,\min}^2(1-\gamma^n )^2(1-\gamma)^2}\right).
\end{align}
These two results have the same dependency on $\epsilon$, $n$, and $1/(1-\gamma)$, but have two differences. First is that in Eq. (\ref{eq:sc:chen}) there is $\mathcal{K}_{S,\min}^{-2}$ while we have $\mathcal{K}_{SA,\min}^{-2}$. This is because we are evaluating the $Q$-function while \cite{chen2021lyapunov} studies policy evaluation for the $V$-function. Another difference is that in Eq. (\ref{eq:sc:chen}) there is an additional factor of $|\mathcal{S}|^{1/2}$. This is because we find the sample complexity to obtain $\mathbb{E}[\|\cdot\|_\infty]\leq \epsilon$ while \cite{chen2021lyapunov} finds the sample complexity to achieve $\mathbb{E}[\|\cdot\|_2]\leq \epsilon$.

\subsection{Proof of Theorems \ref{co:Qpi} to \ref{co:Qtrace}}
The results are obtained by directly applying Theorem \ref{thm:main}. 

\subsection{Computing the Sample Complexity of $Q$-Trace from \citep{chen2021finite}}
To compute the sample complexity of the $Q$-trace algorithm from \citep{chen2021finite}, we will adopt the notation from this paper for consistency. In view of \cite[Theorem 2.1]{chen2021finite}, to obtain $\mathbb{E}[\|Q_k-Q^{\pi,\rho}\|_\infty]\leq 
\epsilon$, we need
\begin{align*}
    \alpha &\sim \mathcal{O}\left(\frac{\epsilon^2}{\log(1/\epsilon)}\right)\Tilde{\mathcal{O}}\left(\frac{\mathcal{K}_{SA,\min}^2f(\gamma D_{c,\min})^2(1-\gamma D_{\rho,\max})^2(1-\gamma)^2}{n f(\gamma \bar{c})^2(\gamma \bar{\rho}+1)^2}\right),
\end{align*}
which implies
\begin{align*}
    k&\sim \mathcal{O}\left(\frac{\log(1/\epsilon)^2}{\epsilon^2}\right)\Tilde{\mathcal{O}}\left(\frac{n f(\gamma \bar{c})^2(\gamma \bar{\rho}+1)^2}{\mathcal{K}_{SA,\min}^3f(\gamma D_{c,\min})^3(1-\gamma D_{\rho,\max})^3(1-\gamma)^2}\right)\\
    &=\Tilde{\mathcal{O}}\left(\frac{\log^2(1/\epsilon)n f(\gamma \bar{c})^2(\gamma \bar{\rho}+1)^2}{\epsilon^2\mathcal{K}_{SA,\min}^3f(\gamma D_{c,\min})^3(1-\gamma D_{\rho,\max})^3(1-\gamma)^2}\right).
\end{align*}
\end{document}